\documentclass{article}
\usepackage[utf8]{inputenc}
\usepackage{amsfonts}
\usepackage{titling}
\usepackage{amsthm}
\usepackage{amssymb}
\usepackage{hyperref}
\usepackage{amsmath}
\usepackage{xcolor}
\usepackage{mathtools}
\usepackage{comment}
\usepackage{listings}
\usepackage{bbm}
\usepackage{caption}
\usepackage{subcaption}
\usepackage{comment}
\usepackage{graphicx}

\usepackage{titlesec}
\titleformat*{\section}{\large\bfseries}
\titleformat*{\subsection}{\bfseries}

\DeclareMathOperator{\sign}{sign}

\newcommand{\bs}{\boldsymbol}
\newcommand{\mc}{\mathcal}

\usepackage[margin=1.2in]{geometry}
\usepackage{graphicx}

\newtheorem{theorem}{Theorem}
\newtheorem{corollary}{Corollary}
\newtheorem{lemma}{Lemma}
\newtheorem{definition}{Definition}

\author{Evzenie Coupkova \thanks{ecoupkov@purdue.edu}, Mireille Boutin\thanks{mboutin@purdue.edu}}
\title{On the Rashomon ratio of infinite hypothesis sets}
\date{}

\begin{document}
\maketitle
\vspace{-0.85in}

\section*{abstract}
Given a classification problem and a family of classifiers, the Rashomon ratio measures the proportion of classifiers that yield less than a given loss. Previous work has explored the advantage of a large Rashomon ratio in the case of a finite family of classifiers.
Here we consider the more general case of an infinite family. We show that a large Rashomon ratio guarantees that choosing the classifier with the best empirical accuracy among a random subset of the family, which is likely to improve generalizability, will not increase the empirical loss too much. We quantify the Rashomon ratio in two examples involving infinite classifier families in order to illustrate situations in which it is large. In the first example, we estimate the Rashomon ratio of the classification of normally distributed classes using an affine classifier. In the second, we obtain a lower bound for the Rashomon ratio of a classification problem with Gram matrix $\bs{H}^\infty$ when the classifier family consists of two-layer ReLU neural networks. 
In general, we show that the Rashomon ratio can be estimated using a training dataset along with random samples from the classifier family and we provide guarantees that such an estimation is close to the true value of the Rashomon ratio.

\section{Introduction}
Given a classification problem and a family of classifiers for this problem, the Rashomon set consists of the functions in the family that yield a loss less than a given $\gamma>0$. Introduced by Breiman \cite{breiman2003statistical}, the Rashomon set framework allows one to focus on an entire set of good classifiers, rather than a single best or optimal classifier. Classification problems with a relatively large Rashomon set have interesting properties. For example, if the Rashomon set of a large model family can be shown to intersect with a family of simple models, then a good yet simple classifier can easily be found. This idea was extensively studied in \cite{semenova_rudin} in the context of seeking interpretable classification models.

The Rashomon ratio measures the proportion of functions within the family that are in the Rashomon set. Previous work \cite{semenova_rudin} 
has explored the use of the Rashomon ratio of a finite family of classifiers. We seek to extend these results to the general case of an infinite family of classifiers equipped with a probability measure. 

The first part of this paper is focused on the problem of estimating the Rashomon ratio in this more general case. The key to dealing with infinite classifier families is to choose a probability measure to model the classifier sampling process, realizing that the value of the Rashomon ratio depends on this choice. In some cases, it is possible to compute the value of the Rashomon ratio (of a given classification problem with a given classifier family and a chosen probability measure on the classifier family) analytically. For cases that cannot be handled analytically, we show how one can estimate the Rashomon ratio using a Monte-Carlo method where both the functions of the family and the training samples are picked at random (Theorem \ref{empirical_sampled_ratio_close_to_true_ratio}). 

Seeking examples featuring a demonstrably large Rashomon set, we compute the Rashomon ratio in the case of affine classifiers applied to binary classification problems where the class-conditional densities are Gaussians in $d$ dimensions with identity covariance  matrices. Since the classification obtained with an affine classifier is unchanged when the parameters of the classifier are multiplied by a positive constant, we identify the set of affine classifiers with the points of a sphere in ${\mathbb R}^{d+1}$. We then put a uniform metric on the sphere (though other choices are possible.) We show, analytically, that the Rashomon ratio in this case approaches $1$ as the distance between the Gaussians approaches infinity, reaching a strictly positive minimum value somewhere in-between. Numerical experiments indicate that this minimum value is obtained when the means of the Gaussians are at a distance approximately two from each other.  The minimum value depends on the dimension $d$. For example, when $d=1$ and $d=2$, the minimum value is approximately $0.36$ and  $0.18$, respectively, which is quite substantial.

In another example, we look at the use of a randomly initialized two-layer neural network with ReLU activation functions for classification. To capture the geometry of the classification problem, we look at the Gram matrix $\bs{H}^{\infty}$ and labeling vector $\bs{y}$ associated to the training data. The generalization and optimization properties of the trained classifiers in such a setup were studied in \cite{arora2019finegrained}; we build upon their work to obtain properties of the Rashomon set and the Rashomon ratio. 
We obtain a lower bound on the empirical Rashomon ratio, which we then relate to the true Rashomon ratio. %{\color{purple}
%We use theorem that shows that the empirical loss converges to $0$ with the number of iterations so empirical anchored fits best. (anchored because we are not comparing the empirical loss to the best empirical loss in the class) in the case where the probability function on the classifier family corresponds to choosing the parameters of the network independently following a normal distribution. } {\color{blue} Can I comment out the purple? - yes}
%The probability function we use in this case corresponds to the random model used to initialize the parameters of the network (i.e., independent Gaussians).  
%We obtain a lower bound on the Rashomon ratio. %The bound depends on the dimension of the data, the number of nodes in the hidden layer of the neural network, the smallest eigenvalue of $\bs{H}^{\infty}$ and $\bs{y}^T(\bs{H}^{\infty})^{-1}\bs{y}$, where $\bs{y}$ is a vector containing the labels of the points. 
The bound depends on the dimension of the data, the number of nodes in the hidden layer of the neural network, the smallest eigenvalue of $\bs{H}^{\infty}$ and $\bs{y}^T(\bs{H}^{\infty})^{-1}\bs{y}$. %, where $\bs{y}$ is a vector containing the labels of the points. 
We illustrate this result using the Iris data set as a training set. 
We compute $\bs{y}^T(\bs{H}^{\infty})^{-1}\bs{y}$ and use it in the formula for the lower bound of the Rashomon ratio of a two-layer neural network. The lower bound for the Rashomon ratio in this particular setup is of the order of $10^{-10}$, which is encouraging when compared to the results from \cite{semenova_rudin} where larger Rashomon ratios were of the order of $10^{-39}$. A potentially higher value could be obtained by using a tighter bound.  

%\color{blue} We computed the dataset complexity measured via $\bs{y}^T(\bs{H}^{\infty})^{-1}\bs{y}$ of the Iris dataset and used this quantity in the formula for the lower bound of the Rashomon ratio of a two-layer neural network setting the probability distribution on the parameter family to be independent Gaussians with increasing amount of noise. \color{red} Not sure if we mention it in the introduction:\color{blue}  The lower bound for the Rashomon ratio in this particular setup is of the order of $10^{-10}$, which is encouraging when compared to the results from \cite{semenova_rudin} where larger Rashomon ratios were of the order of $10^{-39}$.\color{black}

%In the second part of the paper we explore how a large Rashomon ratio of a family of classifiers $\mathcal{F}_2$ can be used to justify the restriction to a subset of classifiers $\mathcal{F}_1 \subset \mathcal{F}_2$. 
In the second part of the paper, we explore how a large Rashomon ratio for a family of classifiers $\mathcal{F}_2$ can be used to guarantee that
the restriction to a subset of classifiers $\mathcal{F}_1 \subset \mathcal{F}_2$, which is likely to improve the generalizability of the trained classifier, will not too severely affect its accuracy. More specifically, we look at the case where $\mathcal{F}_1$ is chosen from $\mathcal{F}_2$ at random following the probability measure on $\mathcal{F}_2$. We quantify the relation between the number of samples drawn, the Rashomon ratio and the probability that $\mathcal{F}_1$ intersects with the Rashomon set (Theorem \ref{theorem_6}), which allows us to relate the empirical accuracy of the best classifier within $\mathcal{F}_1$ to the true accuracy  of the best classifier within $\mathcal{F}_2$ (Theorem \ref{theorem_5}). 
Thus a good classifier can be obtained simply by picking the best classifier (i.e., the one with the best empirical performance) among a small set of randomly chosen ones.

We look at the implications of these theorems for the Rashomon ratio estimates obtained in the first part of the paper. More specifically, we substitute our estimates of the Rashomon ratio (or its lower bound) in Theorem \ref{theorem_6} to see how large a random subset has to be in order to satisfy the assumptions of Theorem \ref{theorem_5}; this way, one can expect a similar accuracy as when using the whole class of functions (before random sampling) while benefiting from the low population loss of a smaller subset. 
In the case of the affine classifiers applied to a mixture of Gaussians, our conclusions relate to the TARP method (thresholding after random projection) \cite{boutin_coupkova}. In that case, the value of the Rashomon ratio provides a clear understanding of the number of random projections required in order to obtain a population loss within a certain range of the optimal population loss.
In the case of ReLU networks, we go back to the Iris dataset example and use the lower bound on the Rashomon ratio previously obtained to obtain an upper bound on the number of randomly generated ReLU networks needed to obtain an accuracy comparable to that of the entire ReLU network family. Again, this bound could be decreased if a tighter bound for the Rashomon ratio was obtained.

\section{Related work}

As mentioned previously, our work builds on previous work by Semenova et al. \cite{semenova_rudin}, but our motivation is different: instead of showing that a large Rashomon ratio is related to the existence of accurate but simple models, we seek to show that a large Rashomon ratio allows for less training when learning. Our results concern an extreme case where the training is reduced to picking classifiers at random. This improves the generalization error (as explored in \cite{boutin_coupkova}) and reduces the computational cost. A large Rashomon ratio may provide guarantees for other learning methods that use random projection such as Random-projection Ensemble Classification in \cite{cannings2017randomprojection}. 
In the field of Neural Networks there is work that shows that a particular learning problem can be solved accurately with an untrained neural network.  Two such examples are visual number sense and face detection in \cite{kim2021visual} and \cite{baek2021face}. This phenomenon may be related to the size of the Rashomon set with respect to a particular probability measure. But we are not aware of any work exploring that relationship.

The Rashomon ratio is also studied in \cite{semenova2023a}. For example, it is shown there that in the case of a dataset with a binary feature matrix, the Rashomon ratio is larger for decision trees of smaller depth. The authors also explore how the Rashomon ratio changes when noise is added to the data. In particular, they show that adding noise to the normally distributed error of the ridge regression problem increases the value of the Rashomon ratio. In our work we explore how adding noise to the data changes the Rashomon ratio in the case of normally distributed data (i.e., by changing the signal to noise ratio through increasing the distance between the means of the Gaussian). We  also explore adding noise to the distribution that generates classifiers for the case of two-layered neural networks applied to a dataset with a given Gram matrix (i.e., by changing the parameter $\kappa$).

The Rashomon set is studied in other works. For example, in \cite{fisher2019all}, the Rashomon set is used to explore the importance of variables and other statistics of an entire set of good models. In \cite{d2022underspecification}, multiple examples of underspecification in modern machine learning models possibly connected to a large Rashomon sets are presented in fields such as computer vision, medical imaging and natural language processing. Close to optimal models are used in \cite{srebro2010smoothness} to provide a tighter generalization bound. Different models within the Rashomon set can have conflicting predictions, as discussed in \cite{madras2019detecting}.

Our focus in this paper is different: we are showing that it is possible to find a model with good accuracy at random as long as the Rashomon set is large. Showing that the Rashomon ratio is large may also have implications for methods which do not necessarily involve randomness. All of the results of this work apart from Theorem \ref{theorem_6} hold for an arbitrary restriction of a large hypothesis set $\mc{F}$ to a smaller subset $\mc{F}_1\subset \mc{F}$. Therefore these results may be useful for methods such as early stopping (\cite{raskutti2014early}), dropout (\cite{srivastava2014dropout}) or enforcing sparsity (\cite{liu2017learning}).

A slightly modified version of this paper was published as a Chapter in a Thesis \cite{mythesis}.

\section{Preliminaries}
\subsection{Notation}
Let us consider a learning problem on $d$-dimensional data $\bs{x} \in \mathbb{R}^d$ by a family of classifiers $\mathcal{F}$. Training is based on a labeled set $\mathcal{S} = \{(\bs{x}_i,y_i)\}_{i=1}^n$ that consists of $n$ sample points, where $\bs{x}_i \in \mathbb{R}^d$ and $y_i\in \mathbb{R}$ and each pair $(\bs{x}_i, y_i)$ is drawn i.i.d. from the distribution $\mathcal{D}$ with probability measure $\nu(\bs{x},y)$.

We define the population loss (true risk) of this learning problem as
\begin{align*}
    L(f) = \int_{\bs{x} \in \mathbb{R}^d, y\in \mathbb{R}} l(f(\bs{x}), y)\nu(d\bs{x},dy),
\end{align*}
and the empirical loss (training risk) of this learning problem as 
\begin{align*}
    \hat{L}(f) =  \frac{1}{n}\sum_{i=1}^{n}l(f(\bs{x}_i), y_i),
\end{align*}
where $l$ is a loss function, for example a 0/1 loss: $l(f(\bs{x}),y)=\mathbbm{1}(f(\bs{x})\neq y)$.

Learning is done over a hypothesis set $\mc{F}$. We denote the VC dimension of the family $\mc{F}$ by $d_{VC}(\mc{F})$. In this work we consider using a smaller subset of $\mc{F}$ drawn from $\mc{F}$ at random. We work with a probability space $(\mc{F}, \sigma(\mc{F}), \rho)$, where $\mc{F}$ is a set of functions that we use for learning, $\sigma(\mc{F})$ is a sigma algebra on $\mc{F}$ and $\rho$ is a measure on ${\sigma(\mc{F})}$.

Most often we consider a case where there exists an equivalence between $\mathcal{F}$ and Euclidean space $\mathbb{R}^{d}$ for some $d\in \mathbb{N}_{+}$. In this particular case the sigma algebra $\sigma(\mc{F})$ is the $\sigma$-algebra of Lebesgue measurable subsets of $\mathbb{R}^d$. Another case is when $\mc{F}$ is a finite set. In this case we consider a $\sigma$-algebra that consists of all subsets of $\mc{F}$.

We form a finite subset of $N$ functions $\{f_1,...,f_N\}$ that are drawn independently from $\mc{F}$ according to the probability measure $\rho$.

We use bold font to denote vectors or matrices. For a matrix $\bs{A}$ and for certain numbers $i,j \in \{1,...,n\}$ $\bs{A}_{ij}$ stands for the entry on $i$-th row and $j$-th column of the matrix $\bs{A}$. We denote Euclidean distance by $\|.\|$ and the Frobenius norm by $\|.\|_F$.

\subsection{Rashomon set, Rashomon ratio}
Let us consider the framework from \cite{semenova_rudin}.  In classification tasks, we implicitly look for a model - i.e.~a function from $\mathcal{F}$ - that has a small loss in some sense. The Rashomon set $\mathcal{R}_{\text{set}}(\gamma)$ is the set of classifiers from $\mathcal{F}$ that have a loss smaller than some $\gamma>0$. The loss of a classifier can be defined in multiple ways, hence there are multiple definitions of the Rashomon set. The Rashomon ratio $\mathcal{R}_{\text{ratio}}(\mathcal{F},\gamma)$ is a number between $0$ and $1$ that quantifies the proportion of the functions that belong to the Rashomon set $\mathcal{R}_{\text{set}}(\mathcal{F}, \gamma)$ within $\mathcal{F}$. The respective volumes of functions are measured using the probability measure $\rho$ on $\mathcal{F}$.  

\begin{definition}[true Rashomon set]\label{true_rashomon_set}
Let $\gamma \in [0,1]$, then
\begin{align*}
    \mathcal{R}_{\text{set}}(\mathcal{F}, \gamma)=\{f \in \mathcal{F}:L(f)\leq \inf_{f\in\mc{F}}L(f)+\gamma\},
\end{align*}
where $L$ is a population loss.
\end{definition}

\begin{definition}[true Rashomon ratio] \label{rashomon_ratio}
Let $\rho$ be a probability measure on $\mathcal{F}$
and let $\gamma \in [0,1]$, then
\begin{equation*}
    \mathcal{R}_{\text{ratio}}(\mathcal{F},\gamma)=\int_{f\in\mathcal{F}}\mathbbm{1}({f \in \mathcal{R}_{\text{set}}(\mathcal{F}, \gamma)})\rho(df).
\end{equation*}
\end{definition}

In practice we may not know the underlying distribution of the data or we may be unable to explicitly compute the reducible risk (the difference between the true risk and the Bayes risk). Then we may rely on the following variations of the previous definitions.
\begin{definition}[empirical Rashomon set] \label{empirical_rashomon_set}
Let $\gamma \in [0,1]$, then
\begin{align*}
    \hat{\mathcal{R}}_{\text{set}}(\mathcal{F}, \gamma)=\{f \in \mathcal{F}:\hat{L}(f) \leq \inf_{f\in\mc{F}}\hat{L}({f})+\gamma\},
\end{align*}
where $\hat{L}$ is an empirical loss (based on a training dataset $\mathcal{S}$).
\end{definition}

\begin{definition}[empirical Rashomon ratio] \label{empirical_rashomon_ratio}
Let $\rho$ be a probability measure on $\mathcal{F}$
and let $\gamma \in [0,1]$, then
\begin{equation*}
    \hat{\mathcal{R}}_{\text{ratio}}(\mathcal{F},\gamma)=\int_{f\in\mathcal{F}}\mathbbm{1}({f \in \hat{\mathcal{R}}_{\text{set}}(\mathcal{F}, \gamma)})\rho(df).
\end{equation*}
\end{definition}
 The previous definitions relate the error of certain functions to minimal errors such as Bayes error. Alternatively, one can define the same concepts using anchored version of the error.
 
\begin{definition}[anchored versions]\label{anchored}

Let $\rho$ be a probability measure on $\mathcal{F}$. Let $\gamma\in[0,1]$, let $L$ be a population loss given distribution of the data $\nu(\bs{x},y)$ and let $\hat{L}$ be an empirical loss for a given dataset $\mathcal{S}$. Then we define
\begin{itemize}
    \item (true anchored Rashomon set)
    \begin{align*}
    \mathcal{R}_{\text{set}}^{\text{anc}}(\mathcal{F}, \gamma)=\{f \in \mathcal{F}:L(f)\leq \gamma\},
    \end{align*}
    \item (true anchored Rashomon ratio)
    \begin{align*}
        \mathcal{R}_{\text{ratio}}^{\text{anc}}(\mathcal{F}, \gamma) = \int_{f\in \mathcal{F}}\mathbbm{1}(f \in \mathcal{R}_{\text{set}}^{\text{anc}}(\mathcal{F}, \gamma))\rho(df)
    \end{align*}
    \item (empirical anchored Rashomon set)
    \begin{align*}
           \hat{\mathcal{R}}_{\text{set}}^{\text{anc}}(\mathcal{F}, \gamma)=\{f \in \mathcal{F}:\hat{L}(f) \leq \gamma\},
    \end{align*}
    \item (empirical anchored Rashomon ratio)
    \begin{align*}
          \hat{\mathcal{R}}_{\text{ratio}}^{\text{anc}}(\mathcal{F},\gamma)=\int_{f\in\mathcal{F}}\mathbbm{1}({f \in \hat{\mathcal{R}}_{\text{set}}^{\text{anc}}(\mathcal{F}, \gamma)})\rho(df).
    \end{align*}
\end{itemize}

\end{definition}
To simplify the following discussion, we shall use the term ``Rashomon ratio" and `Ra\-shomon set" to refer to any of the previously defined Rashomon ratios and Rashomon sets. When the text refers to a specific one (empirical/true and anchored/not), this shall be made clear.

\subsection*{Notes about the probability measure $\rho$ }
The definition of the Rashomon ratio assumes the existence of a probability measure $\rho$ on $\mathcal{F}$.

A probability measure $\rho$ hence defines a measure on the space of functions $\mathcal{F}$.
The value of the Rashomon ratio depends on the choice of $\rho$ and this choice is, to a large extent, arbitrary. One can adapt the measure to the classification problem at hand so to increase the Rashomon ratio. This, in turn, will determine a way to select random samples from $\mathcal{F}$ so to have a higher probability of hitting samples from the Rashomon set. For example, the threshold separating two normally distributed classes in ${\mathbb R}$ is restricted to a compact interval containing both of the class means in \cite{semenova_rudin}. More specifically, $\rho$  is set to be a uniform distribution on the chosen interval. Naturally, increasing the size of the interval, or moving it side to side, changes the value of the Rashomon ratio.

When $\mathcal{F}$ is finite, a straightforward choice is to use a uniform distribution, i.e.~weighing all elements of $\mathcal{F}$ equally with a probability $\frac{1}{\mid \mathcal{F} \mid }$. This is the approach taken in \cite{semenova_rudin}. In this case, we have $\rho(f)=\frac{1}{\mid \mathcal{F}\mid } \sum_{f'\in \mathcal{F}} \delta(f-f')$. It is also possible to weigh certain functions inside $\mathcal{F}$ more heavily than others, as suggested in \cite{semenova2023a}. For discrete but infinite families $\mathcal{F}$, this is essential since weighing all functions equally would lead to a non-valid distribution with $\int_{f\in \mathcal{F}} \rho (df) =\infty$. When the training of a  classifier involves the choice of a random samples from ${\mathcal F}$, then it makes sense to use the same measure as the one used to generate the samples, as we do in Section \ref{twolayernn} for the case of a two-layer neural network initialized with random weights. This allows us to connect the initial guess and the Rashomon set. In general, a large Rashomon ratio is associated to a  large probability of obtaining a classifier with good accuracy by picking among a set of classifiers randomly chosen according to the probability measure $\rho$, as we show in Section \ref{section:advantage}.

 Sometimes a reparametrization of $\mc{F}$ allows us to use $\rho$, which is uniform. For example, if the functions in $ \mathcal{F}$ are parameterized by real-valued parameters \mbox{$(a_1,\ldots, a_d)\in {\mathbb R}^d$}, then a natural choice is to use the Lebesgue measure on $ \mathcal{F}$. However, if the set of possible parameters is unbounded (in the Euclidean sense), such $\rho$ can not be used to compute the Rashomon ratio in the sense of Definition \ref{rashomon_ratio}. Noticing equivalencies can be helpful in such cases. For example, an affine classifier $(a_0, a_1,\ldots, a_d)$ on $ {\mathbb R}^d$ is unchanged by multiplication with a positive constant. Thus, affine classifiers can be associated to points on a sphere in $ {\mathbb R}^{d+1}$, which is a compact set. Thus it is possible to weigh all affine classifiers ``equally" on the sphere, as we do in Section \ref{affine_gauss}, even though there are infinitely many of them.

\section{Numerical estimation of Rashomon ratio - guarantees} \label{num_guarantees}
The Rashomon ratio can be difficult to compute analytically.
An alternative is to use a method from a Monte-Carlo family to approximate it.
For example, one can can estimate  the Rashomon ratio numerically by generating $N$ sample functions from $\mc{F}$ independently according to $\rho$: the proportion of classifiers whose error is less than $\gamma$ is an approximation of a Rashomon ratio. When using the training error, the Rashomon ratio obtained is the empirical Rashomon ratio. Similarly, when using the population error without subtracting Bayes error, the Rashomon ratio obtained is the true anchored Rashomon ratio.

In the following analysis, the set of functions $\mc{F}$ for which we compute the Rashomon ratio is fixed. To simplify the notation, we will drop the symbol $\mc{F}$ from all of the expressions for the Rashomon ratio, so for example $\mc{R}_{\text{set}}({\mc{F}, \gamma})$ will be written as $\mc{R}_{\text{set}}({\gamma})$. Then, the aforementioned Rashomon ratio estimate can be expressed mathematically as: 
\begin{align}\label{approx_ratio}
    \tilde{\mathcal{R}}_{\text{ratio}}(\gamma,N)=\frac{1}{N}\sum_{i=1}^{N} \mathbbm{1}\left({f_i \in \mathcal{R}_{\text{set}}}(\gamma)\right) \sim \int_{\mc{F}} \mathbbm{1}\left({f \in \mathcal{R}_{\text{set}}}\right)\rho(df).
\end{align} 

If we have access to a formula for the reducible error of a classifier (which is the case for affine classification of a mixture of Gaussians as discussed in Section \ref{affine_gauss}) then for each function in the random draw $\{f_1,...,f_N\}$ we can easily decide if $f_i$ belongs to the true Rashomon set and thus estimate the Rashomon ratio using formula \ref{approx_ratio}. The next lemma quantifies how closely this approximates the true Rashomon ratio.

\begin{lemma} \label{ratio_hoeffding}
For all $\gamma>0$ and $\varepsilon>0$ we have 
    \begin{align*}
    P\left(\mid \tilde{\mathcal{R}}_{\text{ratio}}(\gamma,N)-\mathcal{R}_{\text{ratio}}(\gamma)\mid \geq \varepsilon\right)\leq 2\exp\left(-2N\varepsilon^2\right).
    \end{align*}
\end{lemma}
\begin{proof}
As long as we choose sample functions $f_i\in \mathcal{F}$ independently, the random variables $\mathbbm{1}\left(f_i \in \mathcal{R}_{\text{set}}\right)$ are independent. They are also bounded in the interval $[0,1]$. Therefore we can apply Hoeffding's inequality:
\begin{align*}
    P\left(\mid \frac{1}{N}\sum_{i=1}^{N}X_i-EX\mid  \geq \varepsilon\right) \leq 2\exp\left(-2N\varepsilon^2\right)
\end{align*}
for $X_i = \mathbbm{1}(f_i \in \mathcal{R}_{\text{set}})$ and $EX = E(\mathbbm{1}(f_i \in \mathcal{R}_{\text{set}}))$, which is the true Rashomon ratio $\mc{R}_{\text{ratio}}(\gamma)$. 
\end{proof}

Lemma \ref{ratio_hoeffding} is useful when we approximate the true Rashomon ratio when Bayes error is known for the given classification problem, so that the reducible error of a classifier can be obtained. If Bayes error is unknown, then one can use the true error of the classifier in place of the reducible error and approximate the anchored Rashomon ratio instead. The accuracy of such an approximation can be characterized in a similar way. So we can reformulate Lemma \ref{ratio_hoeffding} for the anchored Rashomon ratio using a very similar proof.

\begin{lemma} \label{anchored_ratio_hoeffding}
For all $\gamma>0$ and $\varepsilon>0$ we have 
    \begin{align*}
    P\left(\mid \tilde{\mathcal{R}}_{\text{ratio}}^{\text{anc}}(\gamma,N)-\mathcal{R}_{\text{ratio}}^{\text{anc}}(\gamma)\mid \geq \varepsilon\right)\leq 2\exp\left(-2N\varepsilon^2\right).
    \end{align*}
\end{lemma}

The situation gets more complicated when we cannot compute or bound neither the true nor the reducible error of a classifier for a given $f$. In such case, on top of a random draw of functions $\{f_1,...f_N\}$ we may rely on a finite dataset $\mc{S}$ (consisting of independent samples that follow an identical distribution $\mc{D}$) and approximate the empirical anchored Rashomon ratio. There are two levels of uncertainty for such an approximation: the one that originates in a random pick of functions from $\mc{F}$ and the one that originates in the random pick of datapoints from $\mc{S}$. Theorem \ref{empirical_sampled_ratio_close_to_true_ratio} gives guarantees for such an approximation.  

More specifically, consider an empirical estimate of the Rashomon ratio based on a dataset $\mathcal{S}$ and a set of sample functions $\{f_1,...,f_N\}$:
\begin{align*}
\tilde{\hat{\mathcal{R}}}_{\text{ratio}}^{\text{anc}}(\gamma,N) = \frac{1}{N}\sum_{i=1}^{N}\mathbbm{1}\left(\hat{L}(f_i)\leq \gamma\right) =  \frac{1}{N}\sum_{i=1}^{N}\mathbbm{1}\left(\frac{1}{n}\sum_{j=1}^n l(f_i(\bs{x}_j))\leq \gamma\right).
\end{align*}
Then we have the following result.
\begin{theorem}\label{empirical_sampled_ratio_close_to_true_ratio}
Let us consider a dataset $\mathcal{S}$ that contains $n$ independent identically distributed samples and a set of $N$ functions sampled independently from $\mathcal{F}$ following the probability measure $\rho$. Let us also assume that the loss function $l$ is bounded by $b$. Then for all $\gamma>0$, $\varepsilon>0$ and $\eta>0$ we have
\begin{align}
    P\left({{\mc{R}}}_{\text{ratio}}^{\text{anc}}(\gamma-\varepsilon)-\eta \leq \tilde{\hat{\mc{R}}}_{\text{ratio}}^{\text{anc}}(\gamma, N) \leq {{\mc{R}}}_{\text{ratio}}^{\text{anc}}(\gamma+\varepsilon)+\eta\right) \geq \nonumber \\
    \geq\left(1-\exp\left(-2\frac{n}{N}\left(\frac{\varepsilon}{b}\right)^2\right)\right)^N\left(1-\exp(-N\eta^2)\right)^2.
\end{align}
\end{theorem}

\begin{proof}
Let us split the dataset $\mc{S}$ into $N$ equisized sets (we can forget some datapoints if $N$ does not divide $n$) and for each function $f_i \in \{f_1,...,f_N\}$ use a different set of data to compute its empirical loss $\hat{L}(f_i) = \frac{1}{n/N}\sum_{j=1}^{n/N}l(f_i(\bs{x}_j),y_j)$. 

For each fixed $f_i\in \{f_1,...,f_N\} $ from Hoeffding's inequality we get that for all $\varepsilon>0$
\begin{align}
    P\left(L(f_i)-\varepsilon\leq \hat{L}(f_i)\leq L(f_i)+\varepsilon\right) \geq 1-2\exp\left(-2\frac{n}{N}\left(\frac{\varepsilon}{b}\right)^2\right).
\end{align}
Due to independence, the probability of the union of these events equals
\begin{align}
    P\left(\bigcup_{i=1}^N \left[L(f_i)-\varepsilon\leq \hat{L}(f_i)\leq L(f_i)+\varepsilon\right]\right) = \prod_{i=1}^N P\left(L(f_i)-\varepsilon\leq \hat{L}(f_i)\leq L(f_i)+\varepsilon\right) 
\end{align}
and can be lower bounded:
\begin{align}
       P\left(\bigcup_{i=1}^N \left[L(f_i)-\varepsilon\leq \hat{L}(f_i)\leq L(f_i)+\varepsilon\right]\right) \geq \left(1-2\exp\left(-2\frac{n}{N}\left(\frac{\varepsilon}{b}\right)^2\right)\right)^N.
\end{align}
When $L(f_i)-\varepsilon\leq \hat{L}(f_i)$ then for each $\gamma>0$ if $f_i \in \hat{\mc{R}}_{\text{set}}^{\text{anc}}(\gamma)$ then $f_i \in {\mc{R}}_{\text{set}}^{\text{anc}}(\gamma+\varepsilon)$ also. Since it is true for all $i \in \{1,...N\}$ we can conclude that $\tilde{\hat{\mc{R}}}_{\text{ratio}}^{\text{anc}}(\gamma) \leq \tilde{{\mc{R}}}_{\text{ratio}}^{\text{anc}}(\gamma+\varepsilon)$. Additionally, when $\hat{L}(f_i)\leq L(f_i)+\varepsilon$ then for each $\gamma>0$ if $f_i \in {\mc{R}}_{\text{set}}^{\text{anc}}(\gamma-\varepsilon)$ then $f_i \in \hat{\mc{R}}_{\text{set}}^{\text{anc}}(\gamma)$. And since it is true for all $i \in \{1,...N\}$ we can conclude that $\tilde{{\mc{R}}}_{\text{ratio}}^{\text{anc}}(\gamma-\varepsilon) \leq \tilde{\hat{\mc{R}}}_{\text{ratio}}^{\text{anc}}(\gamma)$. Therefore, an event 
\begin{align}
    \bigcup_{i=1}^N \left[L(f_i)-\varepsilon\leq \hat{L}(f_i)\leq L(f_i)+\varepsilon\right]
\end{align}
implies another event: for all $\gamma>0$
\begin{align}
    \tilde{{\mc{R}}}_{\text{ratio}}^{\text{anc}}(\gamma-\varepsilon, N)\leq\tilde{\hat{\mc{R}}}_{\text{ratio}}^{\text{anc}}(\gamma, N) \leq \tilde{{\mc{R}}}_{\text{ratio}}^{\text{anc}}(\gamma+\varepsilon, N). 
\end{align}
Therefore its probability is no smaller than:
\begin{align}\label{part_1}
    P\left(\tilde{{\mc{R}}}_{\text{ratio}}^{\text{anc}}(\gamma-\varepsilon, N)\leq\tilde{\hat{\mc{R}}}_{\text{ratio}}^{\text{anc}}(\gamma, N) \leq \tilde{{\mc{R}}}_{\text{ratio}}^{\text{anc}}(\gamma+\varepsilon, N)\right) \geq \left(1-2\exp\left(-2\frac{n}{N}\left(\frac{\varepsilon}{b}\right)^2\right)\right)^N.
\end{align}
Using Hoeffding's inequality we can show a one-sided version of Lemma \ref{anchored_ratio_hoeffding} :
\begin{align}\label{part_2}
    P\left({{\mc{R}}}_{\text{ratio}}^{\text{anc}}(\gamma)-\eta \leq \tilde{{\mc{R}}}_{\text{ratio}}^{\text{anc}}(\gamma, N) \right) \geq 1-\exp(-N\eta^2).
\end{align}
and 
\begin{align}\label{part_3}
    P\left(\tilde{{\mc{R}}}_{\text{ratio}}^{\text{anc}}(\gamma+\varepsilon, N) \leq {{\mc{R}}}_{\text{ratio}}^{\text{anc}}(\gamma+\varepsilon)+\eta\right) \geq 1-\exp(-N\eta^2).
\end{align}
Combining \ref{part_1}, \ref{part_2} and \ref{part_3} we get that 
\begin{align}
    P\left({{\mc{R}}}_{\text{ratio}}^{\text{anc}}(\gamma-\varepsilon)-\eta \leq \tilde{\hat{\mc{R}}}_{\text{ratio}}^{\text{anc}}(\gamma, N) \leq {{\mc{R}}}_{\text{ratio}}^{\text{anc}}(\gamma+\varepsilon)+\eta\right) \geq \nonumber \\
    \geq\left(1-\exp\left(-2\frac{n}{N}\left(\frac{\varepsilon}{b}\right)^2\right)\right)^N\left(1-\exp(-N\eta^2)\right)^2.
\end{align}
\end{proof}

\section{Rashomon ratio for affine classifiers applied to a mixture of two Gaussians}\label{affine_gauss}
Let us assume that the distribution $\nu(\bs{x},y)$ generating data points is a mixture of two Gaussian distributions: $\bs{X}\mid(Y=r) \sim \mathcal{N}_d(\bs{\mu}_r, \sigma^2\mathbb{I})$ and $P(Y=r)=\zeta_{r}$ where $r \in \{1,2\}$ and $\zeta_{r}$ are the priors of the corresponding distributions, such that $\zeta_1 + \zeta_{2} = 1$. The optimal decision function for two Gaussians with the same covariance matrix is the following:
\begin{align*}
    g(\boldsymbol{x})=g_1(\boldsymbol{x})-g_2(\boldsymbol{x}) = -\frac{\|\boldsymbol{x}-\boldsymbol{\bs{\mu}_1}\|^2}{2\sigma^2}+\frac{\|\boldsymbol{x}-\boldsymbol{\bs{\mu}_2}\|^2}{2\sigma^2} = -\frac{\bs{x}\cdot(\bs{\mu}_2-\bs{\mu}_1)}{\sigma^2}+\frac{\bs{\mu}_2^2-\bs{\mu}_1^2}{2\sigma^2}.
\end{align*}
When $\boldsymbol{x}$ is in one dimension, $g(x)=0$ when $x=\frac{\mu_1+\mu_2}{2}$, which makes Bayes error equal to 
\begin{align}\label{bayes_error_one_dim}
     1 - \Phi\left(\frac{\mid\mu_2-\mu_1\mid}{2\sigma}\right) =\Phi\left(-\frac{\mid\mu_2-\mu_1\mid}{2\sigma}\right), 
\end{align}
where $\Phi(x)$ stands for the cumulative distribution function of the standard one-dimensional Gaussian distribution with mean $0$ and variance $1$. For the multidimensional case, Bayes error equals
\begin{align}\label{bayes_error}
    1-\Phi\left(\frac{\|\boldsymbol{\mu}_2-\boldsymbol{\mu}_1\|}{2\sigma}\right) = \Phi\left(-\frac{\|\boldsymbol{\mu}_2-\boldsymbol{\mu}_1\|}{2\sigma}\right).
\end{align}

Affine classifier can be viewed as a projection of data to $\mathbb{R}$ followed by a thresholding. After projection, the multivariate gaussian $\boldsymbol{X} \sim \mathcal{N}_d\left(\boldsymbol{\mu}_1, \sigma^2\mathbb{I}\right)$ becomes a univariate gaussian $X_{\boldsymbol{p}} \sim \mathcal{N}_1\left(\boldsymbol{p}\cdot\boldsymbol{\mu}_1, \sigma^2\|\bs{p}\|\right)$. Therefore, after projection, the optimal error \ref{bayes_error} becomes 
\begin{align}\label{error_after_projection}
    1-\Phi\left(\frac{\mid\boldsymbol{p}\cdot\boldsymbol{\bs{\mu}}_2-\boldsymbol{p}\cdot\boldsymbol{\bs{\mu}}_1\mid}{2\sigma\|\bs{p}\|}\right) = \Phi\left(-\frac{\mid\boldsymbol{p}\cdot(\boldsymbol{\mu}_2-\boldsymbol{\mu}_1)\mid}{2\sigma\|\bs{p}\|}\right). 
\end{align}
The reducible error is therefore equal to:
\begin{align} \label{excess_error}
    \Phi\left(-\frac{\mid\boldsymbol{p}\cdot(\boldsymbol{\mu}_2-\boldsymbol{\mu}_1)\mid}{2\sigma\|\bs{p}\|}\right)-\Phi\left(-\frac{\|\boldsymbol{\mu}_2-\boldsymbol{\mu}_1\|}{2\sigma}\right) =  \Phi\left(\frac{\|\boldsymbol{\mu}_2-\boldsymbol{\mu}_1\|}{2\sigma}\right)-\Phi\left(\frac{\mid\boldsymbol{p}\cdot(\boldsymbol{\mu}_2-\boldsymbol{\mu}_1)\mid}{2\sigma\|\bs{p}\|}\right).
\end{align}

Let us define the family of affine classifiers in $\mathbb{R}^d$: \mbox{$\mathcal{F}_{\text{af}} = \{\sign(\bs{p}\cdot \bs{x}+t), \bs{p}\in \mathbb{R}^d, t \in \mathbb{R}\}$}, 

\begin{lemma} \label{excess_linear_mixture}
The reducible error of an affine classifier \mbox{$\sign(\bs{p}\cdot\bs{x}+t) \in \mathcal{F}_{\text{af}}$} is
\begin{align}\label{excess_error_multi}
\begin{split}
    \mathcal{E}_{\bs{\mu}_1,\bs{\mu}_2,\sigma}(\bs{p},t) = \Phi\left(\frac{\|\bs{\mu}_2-\bs{\mu}_1\|}{2\sigma}\right)-\zeta\Phi\left(\frac{\max(\bs{p}\cdot\bs{\mu}_1,\bs{p}\cdot \bs{\mu}_2)-t}{\sigma\|\bs{p}\|}\right)-\\
    -(1-\zeta)\Phi\left(\frac{t-\min(\bs{p}\cdot\bs{\mu}_1,\bs{p}\cdot\bs{\mu}_2)}{\sigma\|\bs{p}\|}\right),
    \end{split}
\end{align}
where $\zeta = \zeta_1$ if $\max(\bs{p}\cdot\bs{\mu}_1, \bs{p}\cdot \bs{\mu}_2)= \bs{p}\cdot\bs{\mu}_1$, otherwise $\zeta=\zeta_2$.
\end{lemma}
\begin{proof}
    The difference between the optimal error of the original high-dimensional problem (\ref{bayes_error}) and its projected version (\ref{error_after_projection}) is equal to 
    \begin{align*}
         \Phi\left(\frac{\|\bs{\mu}_2-\bs{\mu}_1\|}{2\sigma}\right)-\Phi\left(\frac{\mid\bs{p}\cdot(\bs{\mu}_2-\bs{\mu}_1)\mid}{2\sigma\|\bs{p}\|}\right).
    \end{align*}
    We make an additional error by choosing a threshold after projection. That error is equal to 
    \begin{align*}
        \Phi\left(\frac{\bs{p}\cdot(\bs{\mu}_2-\bs{\mu}_1)}{2\sigma\|\bs{p}\|}\right)-\zeta\Phi\left(\frac{\max(\bs{p}\cdot\bs{\mu}_1, \bs{p}\cdot \bs{\mu}_2)-t}{\sigma\|\bs{p}\|}\right)-(1-\zeta)\Phi\left(\frac{t-\min(\bs{p}\cdot\bs{\mu}_1, \bs{p}\cdot \bs{\mu}_2)}{\sigma\|\bs{p}\|}\right).
    \end{align*}
    If we add these two errors together, we get the statement of the lemma. 
\end{proof}

The set of affine classifiers $\mathcal{F}_{\text{af}}$ is a compact set that corresponds to a unit hypersphere $S^d$ in $\mathbb{R}^{d+1}$, because we can rescale the parameters $\bs{p}, t$ so that they lay on the hypersphere without changing the value of the classifier ($(\bs{p},t)$ stands for concatenating the vector $\bs{p}$ with the number $t$):
\begin{align*}
    f(\bs{x}) = \sign(\bs{p}\cdot \bs{x} +t) = \sign\left(\dfrac{\bs{p}\cdot \bs{x} +t}{\|(\bs{p},t)\|}\right).
\end{align*}
In Figure \ref{one_dim_circle} we give some examples of the Rashomon sets (in blue) with varying distances between the means $(2\mu)$ for the affine classifiers applied to a mixture of two Gaussians. Each point on the hypersphere (in one-dimensional case the hypersphere corresponds to the circle) represents a randomly chosen classifier.

\begin{figure}[h]
     \centering
     \begin{subfigure}{0.3\textwidth}
         \includegraphics[width=\textwidth]{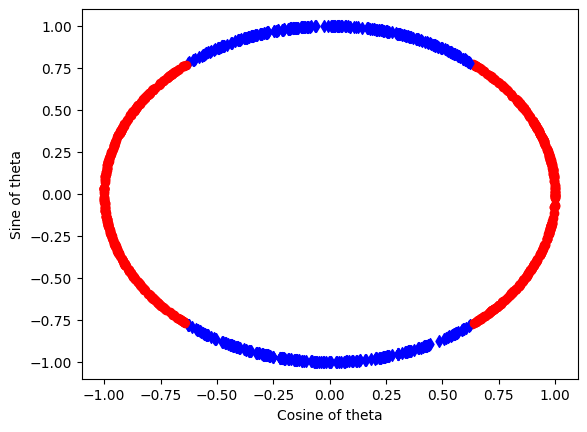}
         \caption{$2\mu = 1$, $\tilde{\mathcal{R}}_{\text{ratio}} = 0.438$}
         \label{rashomon_set_1}
     \end{subfigure}
     \begin{subfigure}{0.3\textwidth}
         \includegraphics[width=\textwidth]{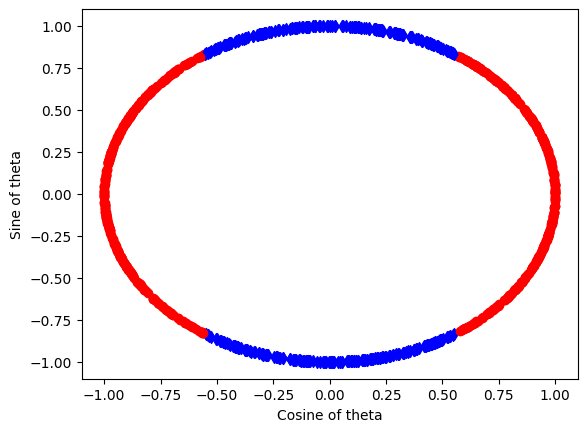}
         \caption{$2\mu = 2$, $\tilde{\mathcal{R}}_{\text{ratio}} = 0.375$}
         \label{rashomon_set_2}
     \end{subfigure}
    \begin{subfigure}{0.3\textwidth}
         \includegraphics[width=\textwidth]{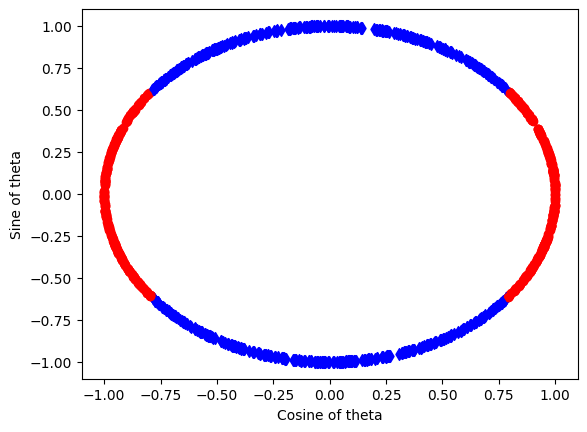}
         \caption{$2\mu = 5$, $\tilde{\mathcal{R}}_{\text{ratio}} =0.581$}
         \label{rashomon_set_5}
     \end{subfigure}
    \begin{subfigure}{0.3\textwidth}
         \includegraphics[width=\textwidth]{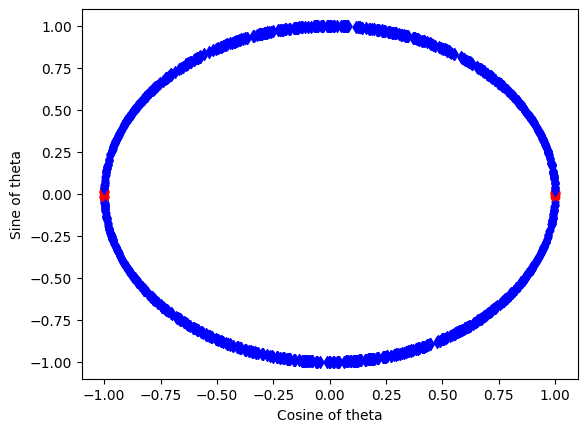}
         \caption{$2\mu = 100$, $\tilde{\mathcal{R}}_{\text{ratio}} =0.989$}
         \label{rashomon_Set_100}
     \end{subfigure}
     \caption{Rashomon set for one-dimensional classification of a mixture of Gaussians by affine functions. Points on the circle correspond to functions in $\mathcal{F}_{\text{af}}$: blue diamonds are functions that belong to the Rashomon set with $\gamma=0.05$, red circles represent functions that do not. Each figure contains $1000$ samples from $\mathcal{F}_{\text{af}}$, the subfigures differ in the distance between the means ($2\mu$), while the other parameters are fixed: $d=1$, $\sigma=1$.} 
     \label{one_dim_circle}
\end{figure}

\subsection{One-dimensional case}
If $d=1$ we can express the reducible error using just one parameter $\theta \in (-\zeta, \zeta]$. Setting $p = \sin(\theta)$ and $f = \cos{\theta}$, the classifier from $\mathcal{F}_{\text{af}}^{c}$ takes the following form
\begin{align*}
    f(x) = \sign(\sin(\theta)x+\cos(\theta)).
\end{align*}
For any $\mu_1,\mu_2 \in \mathbb{R}$ the reducible error of $f$ then becomes 
\begin{align}\label{excess_error_one}
    \mathcal{E}_{\mu_1,\mu_2,\sigma}(\theta) = \Phi\left(\frac{\mu_2-\mu_1}{2\sigma}\right)-\zeta\Phi\left(\frac{\sin(\theta)\mu_2-\cos(\theta)}{\sin(\theta)\sigma}\right)-(1-\zeta) \Phi\left(\frac{\cos(\theta)-\sin(\theta)\mu_1}{\sin(\theta)\sigma}\right).
\end{align}
If we additionally assume that $-\mu_1=\mu_2=\mu$ (antipodal case), the reducible error can be expressed as
\begin{align} \label{zero_com}
    \mathcal{E}_{\mu, \sigma}(\theta)=\Phi\left(\frac{\mu}{\sigma}\right)-\zeta\Phi\left(\frac{\sin(\theta)\mu-\cos(\theta)}{\sin(\theta)\sigma}\right)- (1-\zeta)\Phi\left(\frac{\cos(\theta)+\sin(\theta)\mu}{\sin(\theta)\sigma}\right).
\end{align}

We now explore the analytical properties of the reducible error and Rashomon ratio for the one-dimensional antipodal case.  Let us assume that we have a uniform distribution on the compact set $\mathcal{F}_{af}$ (i.e.~the unit circle). In this case the optimal classifier belongs to $\mc{F}_{af}$, therefore $\inf_{f\in \mc{F}}L(f)$ is the Bayes error for this classification problem and the Rashomon ratio can be expressed as 
\begin{align}
\begin{split}
    \mathcal{R}_{\text{ratio}}(\mu, \sigma, \gamma) = \\
    =\frac{1}{\pi}\int_{0}^{\pi}\mathbbm{1}\left(\Phi\left(\dfrac{\mu}{\sigma}\right)-\zeta\Phi\left(\dfrac{\sin(\theta)\mu-\cos(\theta)}{\sin(\theta)\sigma}\right)-(1-\zeta)\Phi\left(\dfrac{\cos(\theta)+\sin(\theta)\mu}{\sin(\theta)\sigma}\right)\leq\gamma\right)d\theta.
\end{split}
\end{align}

\begin{lemma}(properties of the Rashomon ratio for the case of a mixture of one-dimen\-sional Gaussians)\label{properties_rashomon_Ratio}
    In the case of affine classifiers applied to an antipodal mixture of two one-dimensional Gaussians with means $\pm \mu$, variance $\sigma^2$ and equal priors: $\zeta=\zeta_1=\zeta_2=1/2$ we have that $\forall \gamma>0, \; \sigma>0$ the function $\mathcal{R}_{\text{ratio}}(\mu) = \mathcal{R}_{\text{ratio}}(\mu, \sigma, \gamma)$ as a function of the distance between means, has the following properties:
\begin{itemize}
    \item[(i)] $\mathcal{R}_{\text{ratio}}(\mu)$ is bounded in an interval $[0,1]$ and if $\mu=0$ then $\mathcal{R}_{\text{ratio}}(0)=1$,
    
    \item[(ii)] $\mathcal{R}_{\text{ratio}}(\mu)$ is continuous in $(0,\infty)$,
    
    \item[(iii)] $\lim_{\mu\rightarrow \infty}\mathcal{R}_{\text{ratio}}(\mu)=1$,
    
    \item[(iv)] $\mathcal{R}_{\text{ratio}}(\mu)$ attains its global minimum in $(0, \infty)$,
    
    \item[(v)] the global minimum of $\mathcal{R}_{\text{ratio}}(\mu)$ is strictly larger than $0$.
\end{itemize}
\end{lemma}

\begin{proof}
    \begin{itemize}
    \item[(i)] By the definition of the Rashomon ratio on a compact set we have 
    \begin{align*}
        \mathcal{R}_{\text{ratio}}(\mu, \mathcal{F},\gamma) = \dfrac{\mathcal{V}(\mathcal{R}_{\text{set}}(\mu,\mathcal{F}, \gamma))}{\mathcal{V}(\mathcal{F})}
    \end{align*}
    and since $\{\} \subset \mathcal{R}_{\text{set}}(\mu,\mathcal{F},\gamma)\subset \mathcal{F}$
    we get that $0\leq\mathcal{R}_{\text{ratio}}(\mu)\leq 1$. When $\mu=0$, the reducible error
    \begin{align*}
        \mathcal{E}_{\mu,\sigma}(\theta) = \Phi\left(\frac{\mu}{\sigma}\right)-\zeta\Phi\left(\frac{\sin(\theta)\mu-\cos(\theta)}{\sin(\theta)\sigma}\right)-(1-\zeta)\Phi\left(\frac{\sin(\theta)\mu+\cos(\theta)}{\sin(\theta)\sigma}\right) = \\
        \Phi(0)-1/2\Phi\left(-\frac{\cos(\theta)}{\sin(\theta)\sigma}\right)-1/2\Phi\left(\frac{\cos(\theta)}{\sin(\theta)\sigma}\right) = 0,
    \end{align*}
    for all $\theta>0$, since $\Phi(-x)=1-\Phi(x)$. Therefore $\mathcal{F} \subset \mathcal{R}_{\text{set}}(\mu,\mathcal{F}, \gamma)$ and $\mathcal{R}_{\text{ratio}}(0)=1$.
    
    \item[(ii)] 
      Due to formula \ref{zero_com} we know that the reducible error $\mathcal{E}_{\mu, \sigma}(\theta)$ is continuous in $\mu$ for $\mu\in [0,\infty)$ and in $\theta$ for $\theta \in (0,\pi)$. For each $\varepsilon>0$, there exists $\delta>0$, such that for every $\bar{\mu} \in (\mu-\varepsilon,\mu+\varepsilon)$ the reducible error is close: $\mid\mathcal{E}_{\mu, \sigma}(\theta)-\mathcal{E}_{\bar{\mu}, \sigma}(\theta)\mid \leq \delta$. Assuming that $\phi$ is the density of the standard normal distribution, the derivative of the reducible error function with respect to $\theta$ is 
    \begin{align*}
        \frac{\partial}{\partial \theta} \mathcal{E}_{\mu, \sigma}(\theta) = \frac{1}{\sin(\theta)^2\sigma}\left(-\zeta\phi\left(\frac{\sin(\theta)\mu-\cos(\theta)}{\sin(\theta)\sigma}\right)+(1-\zeta)\phi\left(\frac{\sin(\theta)\mu+\cos(\theta)}{\sin(\theta)\sigma}\right)\right),
    \end{align*}
    let us substitute $\zeta=1/2$ and we see that if 
    \begin{align*}
        \left(-\frac{1}{2}\phi\left(\frac{\sin(\theta)\mu-\cos(\theta)}{\sin(\theta)\sigma}\right)+\frac{1}{2}\phi\left(\frac{\sin(\theta)\mu+\cos(\theta)}{\sin(\theta)\sigma}\right)\right)<0
    \end{align*}
    than the reducible error is decreasing, otherwise it is increasing. It is equivalent with the following if
    \begin{align*}
         \phi\left(\frac{\sin(\theta)\mu-\cos(\theta)}{\sin(\theta)\sigma}\right)>\phi\left(\frac{\sin(\theta)\mu+\cos(\theta)}{\sin(\theta)\sigma}\right)
    \end{align*}
    then the reducible error is decreasing. The larger the distance of $x$ from zero, the larger is its value measured by the normal density. Therefore, if
    \begin{align*}
        \frac{\sin(\theta)\mu-\cos(\theta)}{\sin(\theta)\sigma} < \frac{\sin(\theta)\mu+\cos(\theta)}{\sin(\theta)\sigma},
    \end{align*}
    which is equivalent with $\cos(\theta)>0$, then the reducible error is decreasing. That means that the reducible error as a function of $\theta$ is decreasing on the interval $(0, \pi/2)$ and increasing on the interval $(\pi/2,\pi)$. Due to the inverse function theorem that means that on both intervals the reducible error has an inverse which is also differentiable and continuous. We can connect these two functions at the point that corresponds to $\theta=\pi/2$ and get a function which might not be differentiable there, but is continuous. Therefore, there exists $\eta>0$, such that $\mid\mathcal{E}_{\mu, \sigma}(\theta)-\mathcal{E}_{\bar{\mu}, \sigma}(\theta)\mid \leq \delta$ implies that $\theta \in (\theta-\eta,\theta+\eta)$. In this particular case the Rashomon set corresponds to an arc on the circle. The angle of the arc can not grow or shrink by more than $\eta$ due to the previous discussion. Therefore the Rashomon ratio can not grow or shrink by more than $\eta/\pi$:
    \begin{align*}
        \mathcal{R}_{\text{ratio}}(\mu\pm\varepsilon, \mathcal{F},\gamma)\in \left[ \mathcal{R}_{\text{ratio}}({\mu}, \mathcal{F},\gamma) -\frac{\eta}{\pi}, \mathcal{R}_{\text{ratio}}({\mu}, \mathcal{F},\gamma) + \frac{\eta}{\pi} \right]
    \end{align*}
    Therefore Rashomon ratio is continuous in $\mu$.
        
    \item[(iii)] The limit of the reducible error when means move apart and $\mu$ goes to infinity is
    \begin{align*}
   &\lim_{\mu\rightarrow\infty}\mathcal{E}(\mu,\sigma, \theta) = \\
   &=\lim_{\mu \rightarrow \infty}\left(\Phi\left(\frac{\mu}{\sigma}\right)-\zeta\Phi\left(\frac{\sin(\theta)\mu-\cos(\theta)}{\sin(\theta)\sigma}\right)- (1-\zeta)\Phi\left(\frac{\cos(\theta)+\sin(\theta)\mu}{\sin(\theta)\sigma}\right)\right)=\\
   &=\lim_{\mu\rightarrow\infty}\left(\Phi\left(\frac{\mu}{\sigma}\right)-1/2\Phi\left(\frac{\mu}{\sigma}-\frac{\cot(\theta)}{\sigma}\right)-1/2\Phi\left(\frac{\mu}{\sigma}+\frac{\cot(\theta)}{\sigma}\right)\right)=0.
\end{align*}
For the Rashomon ratio we have:
\begin{align}\label{limit_rash_ratio_mu}
    \lim_{\mu\rightarrow\infty} \mathcal{R}_{\text{ratio}}(\mu)= \lim_{\mu\rightarrow\infty}\frac{1}{2\pi}\int_{-\pi}^{\pi}\mathbbm{1}\left(\mathcal{E}(\mu, \sigma, \theta)\leq\gamma\right)d\theta.
\end{align}
 We can switch the order of the limit and integral due to Dominated Convergence Theorem and since indicator is a continuous function everywhere apart from one point we get that the right hand side of \ref{limit_rash_ratio_mu} is equal to: 
\begin{align*}
\begin{split}
    \frac{1}{2\pi}\int_{-\pi}^{\pi}\mathbbm{1}\left(\lim_{\mu\rightarrow \infty}\left(\mathcal{E}(\mu, \sigma, \theta)\leq\gamma\right)\right)d\theta =\frac{1}{2\pi}\int_{-\pi}^{\pi} d\theta=1.
\end{split}
\end{align*}
Therefore the Rashomon ratio for any $\gamma>0$ converges to $1$ as $\mu \rightarrow \infty$.
\item[(iv)] If minimum of $\mathcal{R}_{\text{ratio}}(\mu)$ is not attained, then $\inf_{\mu\in(0, \infty)}\mathcal{R}_{\text{ratio}}(\mu) = \lim_{\mu \rightarrow \infty} \mathcal{R}_{\text{ratio}}(\mu)$. Due to (iii) $\lim_{\mu \rightarrow \infty} \mathcal{R}_{\text{ratio}}(\mu) = 1$, therefore there exists at least one point $\mu^{\star} \in [0, \infty)$ such that $\mu^{\star} = \min_{\mu\in [0,\infty)}\mathcal{R}_{\text{ratio}}(\mu) = \inf_{\mu\in [0,\infty)}\mathcal{R}_{\text{ratio}}(\mu)$, which is a global minimum.
 \item[(v)] When $\theta$ converges to $\pi/2$, the reducible error converges to $0$:
 \begin{align*}
 \lim_{\theta \rightarrow \pi/2} \mathcal{E}_{\mu,\sigma} (\theta) =\lim_{\theta \rightarrow \pi/2}  \Phi\left(\dfrac{\mu}{\sigma}\right)-1/2\Phi\left(\dfrac{\sin(\theta)\mu-\cos(\theta)}{\sin(\theta)\sigma}\right)-1/2\Phi\left(\dfrac{\cos(\theta)+\sin(\theta)\mu}{\sin(\theta)\sigma}\right) = \\ =\Phi\left(\dfrac{\mu}{\sigma}\right)-1/2\Phi\left(\dfrac{\mu}{\sigma}\right)-1/2\Phi\left(\dfrac{\mu}{\sigma}\right)=0.
 \end{align*}
 Therefore due to continuity of the reducible error in $\theta$ near $\pi/2$, for every $\gamma>0$ there exists a neighborhood $(\pi/2 - \eta_{\gamma}, \pi/2+\eta_{\gamma})$, such that if $\theta$ belongs to this neighborhood, the reducible error is less than $\gamma$. Therefore for each $\gamma>0$ the Rashomon ratio is strictly larger than $2\eta_{\gamma}/\pi$, its minimum is strictly larger than $0$.
    \end{itemize}
\end{proof}

\begin{figure}[h]
     \centering
     \begin{subfigure}{0.45\textwidth}
         \includegraphics[width=\textwidth]{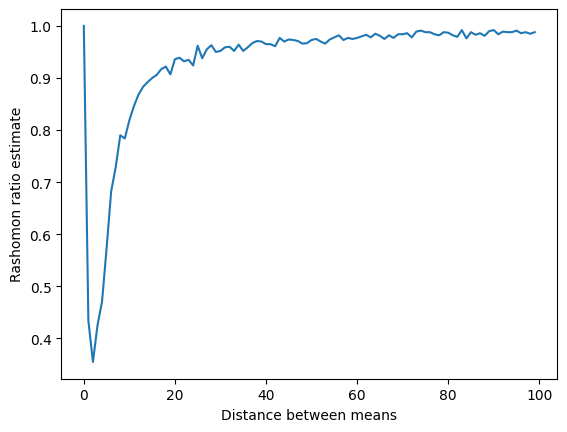}
         \caption{The distance $2\mu \in [0,100]$.}
         \label{ratio_means_wider_one_dim}
     \end{subfigure}
     \begin{subfigure}{0.45\textwidth}
         \includegraphics[width=\textwidth]{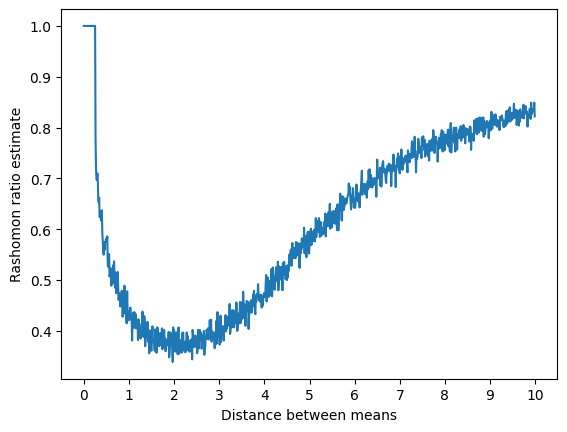}
         \caption{The distance $2\mu \in [0,10]$.}
         \label{ratio_means_narrower_one_dim}
     \end{subfigure}
     \caption{Rashomon ratio as a function of the distance between the means in the Gaussian mixture. The parameters are set to ${d}={1}$, $\sigma=1$ and $\gamma=0.05$. We estimate the Rashomon ratio using $1000$ functions from $\mathcal{F}_{af}$ randomly generated following a Uniform distribution on the circle as in Figure \ref{one_dim_circle}. The proportion of functions whose value according to \ref{zero_com} is smaller than $\gamma$ corresponds to the value of the Rashomon ratio in the plots. The minimum of the function on the left is when the distance between the means is $2\mu = 2$, where Rashomon ratio is approximately $0.355$. The minimum of the function on the right is at $2\mu = 1.97$, where Rashomon ratio is approximately $0.339$. According to Lemma \ref{ratio_hoeffding} these approximations of Rashomon ratio are within an error of $\varepsilon = 0.05$ to the true Rashomon ratio with probability at least $98\%$ since $N=1000$.}
\end{figure}

\subsection{Higher-dimensional case}

In a more general setup when we consider a mixture of Gaussians in $d$ dimensions, an analogous version of Lemma \ref{properties_rashomon_Ratio} can be formulated. Let us assume that the mixture of Gaussians in higher dimensional space is again antipodal: $\bs{\mu}_1=-\bs{\mu}_2=\bs{\mu}$. Then the reducible error is 
\begin{align*}
 \mathcal{E}_{\bs{\mu},\sigma}(\bs{p},t) = \Phi\left(\frac{\|\bs{\mu}\|}{\sigma}\right)-\zeta\Phi\left(\frac{\max(\bs{p}\cdot\bs{\mu},-\bs{p}\cdot \bs{\mu})-t}{\sigma\|\bs{p}\|}\right)-(1-\zeta)\Phi\left(\frac{t-\min(\bs{p}\cdot\bs{\mu},-\bs{p}\cdot\bs{\mu})}{\sigma\|\bs{p}\|}\right),
 \end{align*}
 which is equal to 
 \begin{align*}
 \Phi\left(\frac{\|\bs{\mu}\|}{\sigma}\right)-\zeta\Phi\left(\frac{\mid \bs{p}\cdot\bs{\mu}\mid-t}{\sigma\|\bs{p}\|}\right)-(1-\zeta)\Phi\left(\frac{t+\mid \bs{p}\cdot\bs{\mu}\mid }{\sigma\|\bs{p}\|}\right).
\end{align*}

Here, as in the one-dimensional case, the optimal classifier is a hyperplane and it belongs to $\mc{F}_{af}$, therefore $\inf_{f\in \mc{F}}L(f)$ is the Bayes error of this classification problem and the Rashomon ratio can be expressed as 
\begin{align*}
    \mathcal{R}_{\text{ratio}}(\|\bs{\mu}\|, \sigma, \gamma) = \int_{(\bs{p},t) \in S^{d}}\frac{1}{{\mathcal{V}(S^{d})}} \mathbbm{1}[\mathcal{E}_{\bs{\mu},\sigma}(\bs{p},t) \leq \gamma] d(\bs{p},t).
\end{align*}

The analytic properties of the Rashomon ratio as a function of the distance between the means are similar in the high-dimensional settings as in the one-dimensional one as is illustrated by Figures \ref{dimension_2} and \ref{dimension_10}. We provide Lemma \ref{properties_of_the_rashomon_ratio_multidim} that extends the results of Lemma \ref{properties_rashomon_Ratio} to the multidimensional case. 

\begin{figure}
     \centering
     \begin{subfigure}{0.45\textwidth}
         \includegraphics[width=\textwidth]{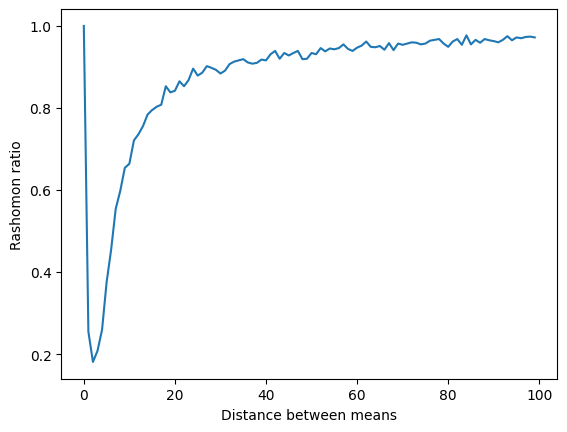}
         \caption{The distance $2\|\bs{\mu}\|\in[0,100]$.}
         \label{ratio_means_wider_two_dim}
     \end{subfigure}
     \begin{subfigure}{0.45\textwidth}
         \includegraphics[width=\textwidth]{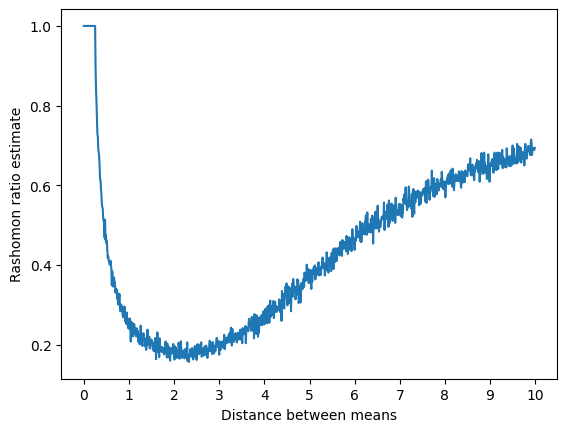}
         \caption{The distance $2\|\bs{\mu}\|\in[0,10]$.}
         \label{ratio_means_narrower_two_dim}
     \end{subfigure}
     \caption{Rashomon ratio as a function of the distance between the means in the Gaussian mixture. The parameters are set to $d=2$, $\sigma=1$ and $\gamma=0.05$. Rashomon ratio is estimated as a proportion of $1000$ functions generated randomly according to a Uniform distribution on a sphere which have an reducible error from (\ref{excess_error_multi}) smaller than $\gamma$. The minimum of the function on the left is when the distance between the means is $2\|\bs{\mu}\| = 2$, where Rashomon ratio is approximately $0.181$. The minimum of the function on the right is at $2\|\bs{\mu}\| = 2.33$, where  the Rashomon ratio is approximately $0.157$. According to Lemma \ref{ratio_hoeffding} these Rashomon ratio estimates are within an error of $\varepsilon = 0.05$ to the true Rashomon ratio with probability at least $98\%$ since $N=1000$.}
     \label{dimension_2}
\end{figure}

\begin{figure}[ht!]
     \centering
     \begin{subfigure}{0.45\textwidth}
         \includegraphics[width=\textwidth]{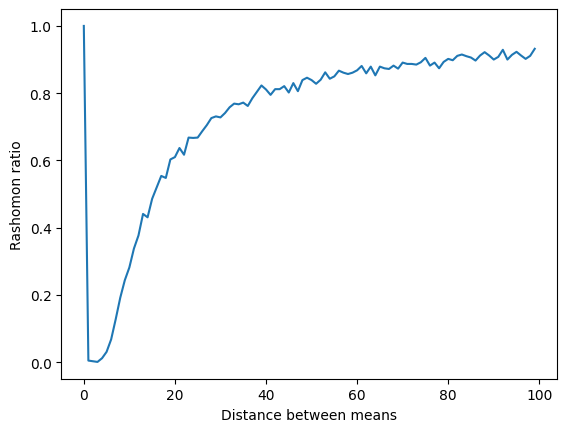}
         \caption{The distance $2\|\bs{\mu}\|\in[0,100]$.}
         \label{ratio_means_wider_ten_dim}
     \end{subfigure}
     \begin{subfigure}{0.45\textwidth}
         \includegraphics[width=\textwidth]{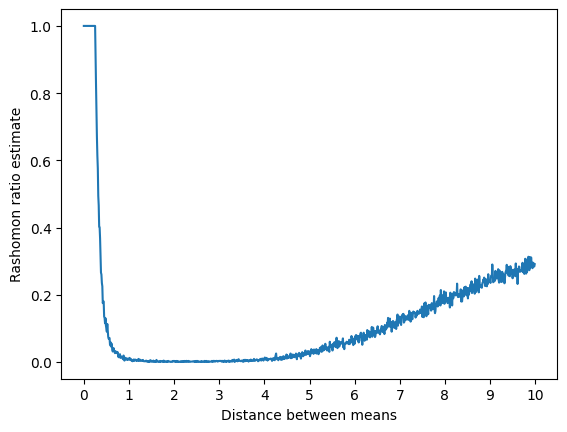}
         \caption{The distance $2\|\bs{\mu}\|\in[0,10]$.}
         \label{ratio_means_narrower_ten_dim}
     \end{subfigure}
     \caption{Rashomon ratio as a function of the distance between the means in the Gaussian mixture. The parameters are set to ${d}={10}$, $\sigma=1$ and $\gamma=0.05$. Rashomon ratio is estimated as a proportion of $1000$ functions generated randomly according to a Uniform distribution on a sphere which have an reducible error from (\ref{excess_error_multi}) smaller than $\gamma$. The minimum of the function on the left is when the distance between the means is  $2\|\bs{\mu}\| = 3$, where Rashomon ratio is approximately $0.001$. The minimum of the function on the right is at $2\|\bs{\mu}\| = 1.43$, where Rashomon ratio is approximately $0$ (according to Theorem \ref{properties_of_the_rashomon_ratio_multidim} Rashomon ratio is never $0$, but since we are making the approximation based on $1000$ functions only we are making a certain error -  according to Lemma \ref{ratio_hoeffding} there is an error less or equal to  $\varepsilon = 0.05$ with probability at least $98\%$ since $N=1000$.}
     \label{dimension_10}
\end{figure}

\newpage

\begin{lemma}\label{properties_of_the_rashomon_ratio_multidim}(properties of the Rashomon ratio for the case of a mixture of multidimensional Gaussians)
Let us consider a distribution consisting of a mixture of two Gaussians in $\mathbb{R}^d$ with means $\bs{\mu}$ and $-\bs{\mu}$, respectively (antipodal case), and covariance matrices $\sigma^2\mathbb{I}_d$ for some $\sigma>0$. Let us assume that the priors are equal: $\zeta_1 = \zeta_2 = 1/2$. Let us also assume that the classifier comes from a family of affine classifiers in $\mathbb{R}^d$ denoted by $\mathcal{R}_{\text{af}}$. Let $\gamma>0$, ratio $\mathcal{R}_{\text{ratio}}(\|\bs{\mu}\|)=\mathcal{R}_{\text{ratio}}(\|\bs{\mu}\|,\sigma,\gamma)$ 
on $\mathcal{F}_{\text{af}}$ (as a function of $\|\bs{\mu}\|$) has the following properties 
 \begin{itemize}
    \item[(i)] $\mathcal{R}_{\text{ratio}}(\|\bs{\mu}\|)$ is bounded on the interval $[0,1]$ and if $\|\bs{\mu}\|=0$ then $\mathcal{R}_{\text{ratio}}(0)=1$,
    
    \item[(ii)] $\mathcal{R}_{\text{ratio}}(\|\bs{\mu}\|)$ is continuous for $\|\bs{\mu}\| \in (0, \infty)$,
    
    \item[(iii)] $\lim_{\|\bs{\mu}\|\rightarrow \infty}\mathcal{R}_{\text{ratio}}(\|\bs{\mu}\|) = 1$,
    
    \item[(iv)] $\mathcal{R}_{\text{ratio}}(\|\bs{\mu}\|)$ attains its global minimum in $(0, \infty)$,
    
    \item[(v)] the global minimum of $\mathcal{R}_{\text{ratio}}(\|\bs{\mu}\|)$ is strictly larger than $0$.
\end{itemize}
\end{lemma}

\begin{proof}
\begin{itemize}
\item[(i)] Rashomon ratio is always bounded between $0$ and $1$ since it signifies the proportion of functions that have a certain property. The reducible error for $\|\bs{\mu}\|=0$ is (note that when $\|\bs{\mu}\|=0$, $\bs{\mu}=\bs{0}$ as well) for any $\bs{p}$ and $t$ is
\begin{align*}
    \mathcal{E}_{\bs{\mu},\sigma}(\bs{p},t) = \Phi\left(\frac{\|\bs{\mu}\|}{\sigma}\right)-\frac{1}{2}\Phi\left(\frac{\max(\bs{p}\cdot\bs{\mu},-\bs{p}\cdot \bs{\mu})-t}{\sigma\|\bs{p}\|}\right)
    -\frac{1}{2}\Phi\left(\frac{t-\min(\bs{p}\cdot\bs{\mu},-\bs{p}\cdot\bs{\mu})}{\sigma\|\bs{p}\|}\right),
    \end{align*}
    which is equal to
    \begin{align*}
    \Phi(0) - 1/2 + 1/2\Phi\left(\frac{t}{\sigma\|\bs{p}\|}\right) - 1/2\Phi\left(\frac{t}{\sigma\|\bs{p}\|}\right) = \frac{1}{2}-\frac{1}{2}=0.
\end{align*}
Therefore every function from $\mathcal{F}_{\text{af}}$ belongs to the Rashomon set and $\mathcal{R}_{\text{ratio}}(0)=1$.

\item[(ii)] 
Due to the Dominated Convergence Theorem we have that 
\begin{align*}
     \lim_{\varepsilon\rightarrow 0} \mc{R}_{\text{ratio}}\left(\|\bs{\mu}\|+\varepsilon\right) = \int_{(\bs{p},t) \in S^{d}}\frac{1}{{\mathcal{V}(S^{d})}} \lim_{\varepsilon\rightarrow 0} \mathbbm{1}[\mathcal{E}_{\bs{\mu}+\varepsilon,\sigma}(\bs{p},t) \leq \gamma] d(\bs{p},t).
\end{align*}
If $\mathcal{E}_{\bs{\mu}+\varepsilon,\sigma}(\bs{p},t) < \gamma$ then 
\begin{align*}
    \lim_{\varepsilon\rightarrow 0} \mathbbm{1}[\mathcal{E}_{\bs{\mu}+\varepsilon,\sigma}(\bs{p},t) \leq \gamma] = \mathbbm{1}[\mathcal{E}_{\bs{\mu},\sigma}(\bs{p},t) \leq \gamma].
\end{align*}
And since the measure of the points $(\bs{p},t)\in S^{d}$ where $\mathcal{E}_{\bs{\mu}+\varepsilon,\sigma}(\bs{p},t) = \gamma$ is zero we get that 
\begin{align*}
\int_{(\bs{p},t) \in S^{d}}\frac{1}{{\mathcal{V}(S^{d})}} \lim_{\varepsilon\rightarrow 0} \mathbbm{1}[\mathcal{E}_{\bs{\mu}+\varepsilon,\sigma}(\bs{p},t) \leq \gamma] d(\bs{p},t) =\\
=\int_{(\bs{p},t) \in S^{d}}\frac{1}{{\mathcal{V}(S^{d})}} \mathbbm{1}[\mathcal{E}_{\bs{\mu},\sigma}(\bs{p},t) \leq \gamma] d(\bs{p},t),
\end{align*}
therefore
\begin{align*}
     \lim_{\varepsilon\rightarrow 0} \mc{R}_{\text{ratio}}\left(\|\bs{\mu}\|+\varepsilon\right) = \mc{R}_{\text{ratio}}\left(\|\bs{\mu}\|\right),
\end{align*}
which means that the Rashomon ratio is a continuous function of the distance between the means. 
\color{black}

\item[(iii)] When $\sigma>0$ is fixed and $\|\bs{\mu}\| \rightarrow \infty$, $\forall \bs{p} \in \mathbb{R}^d \setminus  \bs{0}$ it holds that  $\frac{\mid \bs{p}\cdot \bs{\mu}\mid}{\sigma\|\bs{p}\|} \rightarrow \infty$ as well (due to the reversed Cauchy-Schwartz inequality - Pólya and Szegö’s inequality - that says that there exists a constant $C$ such that $\|\bs{\mu}\|\leq C \bs{p}\cdot(\bs{\mu})/\|\bs{p}\| $). When $\bs{p}$ and $t$ are fixed,
the limit of the reducible error is 
    \begin{align*}
        \forall \bs{p} \in \mathbb{R}^d\setminus\{0\},t\in\mathbb{R}: \lim_{\|\bs{\mu}\|\rightarrow \infty}&\Phi\left(\frac{\|\bs{\mu}\|}{\sigma}\right)-\zeta\Phi\left(\frac{\mid \bs{p}\cdot\bs{\mu}\mid-t}{\sigma\|\bs{p}\|}\right)-\\
        &-(1-\zeta)\Phi\left(\frac{t+\mid \bs{p}\cdot\bs{\mu}\mid }{\sigma\|\bs{p}\|}\right) 
        =1-\zeta-(1-\zeta)=0.
    \end{align*}
    
 Hence
\begin{align}\label{limit_rash_ratio}
    \lim_{\|\bs{\mu}\|\rightarrow \infty} \mc{R}_{\text{ratio}}\left(\|\bs{\mu}\|\right) = \lim_{\|\bs{\mu}\|\rightarrow \infty}\int_{(\bs{p},t) \in S^{d}}\frac{1}{{\mathcal{V}(S^{d})}} \mathbbm{1}[\mathcal{E}_{\bs{\mu},\sigma}(\bs{p},t) \leq \gamma] d(\bs{p},t).
\end{align}
Due to the Dominated Convergence Theorem the right hand side of \ref{limit_rash_ratio} equals
\begin{align}\label{limit_rash_ratio_2}
    \int_{(\bs{p},t) \in S^{d}} \frac{1}{{\mathcal{V}(S^{d})}}\lim_{\|\bs{\mu}\|\rightarrow \infty}\mathbbm{1}[\mathcal{E}_{\bs{\mu},\sigma}(\bs{p},t) \leq \gamma] d(\bs{p},t),
\end{align}
and since the reducible error is always non-negative, it convergences to zero from the right of zero, therefore the right hand side of \ref{limit_rash_ratio_2} equals
\begin{align*}
    \int_{(\bs{p},t) \in S^{d}} \frac{1}{{\mathcal{V}(S^{d})}}\mathbbm{1}[\lim_{\|\bs{\mu}\|\rightarrow \infty}\mathcal{E}_{\bs{\mu},\sigma}(\bs{p},t) \leq \gamma] d(\bs{p},t) = 1.
\end{align*}
   
   \item[(iv)] If minimum of $\mathcal{R}_{\text{ratio}}(\|\bs{\mu}\|)$ is not attained, then
   \begin{align}
   \inf_{\|\bs{\mu}\|\in(0, \infty)}\mathcal{R}_{\text{ratio}}(\|\bs{\mu}\|) = \lim_{\|\bs{\mu}\| \rightarrow \infty} \mathcal{R}_{\text{ratio}}(\|\bs{\mu}\|).
   \end{align}
   Due to (iii) $\lim_{\|\bs{\mu}\| \rightarrow \infty} \mathcal{R}_{\text{ratio}}(\|\bs{\mu}\|) = 1$, therefore there exists at least one point $\|\bs{\mu}\|^{\star} \in [0, \infty)$ such that $\|\bs{\mu}\|^{\star} = \min_{\|\bs{\mu}\|\in [0,\infty)}\mathcal{R}_{\text{ratio}}(\|\bs{\mu}\|) = \inf_{\|\bs{\mu}\|\in [0,\infty)}\mathcal{R}_{\text{ratio}}(\|\bs{\mu}\|)$, which is a global minimum.

    \item[(v)] Let us take the limit of an reducible error, when $\bs{p} \rightarrow \bs{\mu}$ and $t \rightarrow 0$: 
    \begin{align*}
          \lim_{\bs{p} \rightarrow \bs{\mu},\; t \rightarrow 0}\mathcal{E}_{\bs{\mu},\sigma}(\bs{p},f) = \lim_{\bs{p} \rightarrow \bs{\mu},\; t \rightarrow 60} &\Phi\left(\frac{\|\bs{\mu}\|}{\sigma}\right)-
          \frac{1}{2}\Phi\left(\frac{\mid \bs{p}\cdot\bs{\mu}\mid-t}{\sigma\|\bs{p}\|}\right)-\frac{1}{2}\Phi\left(\frac{t+\mid\bs{p}\cdot\bs{\mu}\mid}{\sigma\|\bs{p}\|}\right) = \\
          =&\Phi\left(\frac{\|\bs{\mu}\|}{\sigma}\right)-
          \frac{1}{2}\Phi\left(\frac{\| \bs{\mu}\|^2}{\sigma\|\bs{\mu}\|}\right)-\frac{1}{2}\Phi\left(\frac{\| \bs{\mu}\|^2}{\sigma\|\bs{\mu}\|}\right) = 0 
    \end{align*}
Therefore, since function $\mathcal{E}_{\bs{\mu}, \sigma}(\bs{p},t)$ is continuous, for each $\gamma>0$ there exists an open neighborhood of $(\bs{\mu}, 0)$ in a form of a hypersphere with radius $\eta_{\gamma}$, such that if $(\bs{p},t) \in \mathcal{U}_{\eta_{\gamma}}(\bs{\mu},0)$ the reducible error is larger than $\gamma$. Therefore for all $\gamma>0$ the Rashomon ratio is strictly larger than $\eta_{\gamma}$.
\end{itemize}
\end{proof}

\section{Rashomon ratio of a two-layer neural network applied to dataset with positive-definite Gram matrix} \label{twolayernn}

In this section, we consider the family of classifiers $\mathcal{F}$ consisting of two-layer neural networks with ReLU (rectified linear unit) activation functions. These can be written as 
\begin{align*}
    f_{\bs{W},\bs{a}}(\bs{x}) = \frac{1}{\sqrt{m}} \sum_{r=1}^ma_r\sigma(\bs{w}_r^T\bs{x}),
\end{align*}
where $\bs{x}\in \mathbb{R}^d$, $\bs{w}_1,...\bs{w}_m \in \mathbb{R}^d$ are the weight vectors in the first layer and $a_1,...a_m \in \mathbb{R}$ are the weights in the second layer. Following the setting of \cite{arora2019finegrained}, we assume that the neural network is trained by randomly initiated gradient descent on the quadratic loss over the data $\mc{S}$. More specifically, the parameters are initialized as
\begin{align*}
    \bs{w}_r(0) \sim \mathcal{N}(\bs{0}, \kappa^2 \bs{I}),\; a_r\sim \text{Unif}(\{-1,1\})\;\; \forall r\in [m].
\end{align*}
The second layer is then fixed and the first layer $\bs{W}$ is optimized through Gradient Descent with step size $\eta$ on the following objective function:
\begin{align*}
    \Phi(\bs{W}) = \frac{1}{2} \sum_{i=1}^n \left(y_i-f_{\bs{W,\bs{a}}}(x_i)\right)^2.
\end{align*}
The training is done on a labeled dataset with $n$ samples $\mathcal{S} = \{(\bs{x}_i,y_i)\}_{i=1}^{n}$  chosen i.i.d. following a distribution $\mathcal{D}$. We assume that $\|\bs{x}_i\|=1$ and $\mid y_i\mid \leq 1$.

 We capture the structure of the sample set with the Gram matrix $\bs{H}^{\infty}\in {\mathbb R}^{n\times n}$ as in \cite{arora2019finegrained},
 which is defined as
    \begin{align}
    \begin{split}
    \bs{H}^{\infty}_{i,j} &= E_{\bs{w}\sim \mathcal{N}(\bs{0},\bs{I})} [\bs{x}_i^T\bs{x}_j\mathbbm{1}\{\bs{w}^T\bs{x}_i\geq 0 ,\bs{w}^T\bs{x}_j\geq 0\}]=\\
    &= \frac{\bs{x}_i^T\bs{x}_j(\pi-\arccos(\bs{x}_i^T\bs{x}_j))}{2\pi} \;\;\;\;\; \forall i,j\in [n].
    \end{split}
    \end{align}
Assume that  $\bs{H}^{\infty}$ is positive-definite, so its smallest eigenvalue $\lambda_0$ is positive.

The hypothesis set $\mathcal{F}$ corresponds to the set of the weights in the neural network:
\begin{align*}
    \mc{F} = \{(\bs{W}, \bs{a}), \bs{W} \in \mathbb{R}^{d\times m}, \bs{a} \in \mathbb{R}^m\}.
\end{align*}
The weights $\bs{a}$ for the second layer are fixed after random initialization, and the generalization properties of the final classifier do not depend on the specific values of these weights. Thus we consider $\mathcal{F}^{\bs{a}}$, a subset of $\mathcal{F}$ consisting of all functions with weights $\bs{a}$ fixed. 

Since the rows are generated according to $\bs{w}_r(0)\sim \mathcal{N}(\bs{0}, \kappa^2 \bs{I}_d)$ independently, the initial weights $\bs{W}(0)$ are chosen according to a Gaussian distribution $\mathcal{N}(\bs{0}, \kappa^2\bs{I}_{d\times m})$, where \mbox{$0< \kappa^2 \leq 1$} controls the magnitude of the initialization. This defines a probability measure $\rho_{\bs{a}}(f)$ on $\mathcal{F}_{\bs{a}}$. Using the probability model used to generate $\bs{a}$ (uniform distribution over $\{-1,1\}^m$), this probability measure can be extended to a probability measure $\rho$ on $\mathcal{F}$. In this section,  we compute a lower bound on the (empirical anchored) Rashomon ratio for $\mathcal{F}_{\bs{a}}$ as defined by $\rho_{\bs{a}}(f)$. These bounds hold for all $\bs{a}\in \{ -1, 1\}^{m}$, and are independent of the specific value of $\bs{a}$.~ Therefore they can be extended to bounds on the Rashomon ratio for the entire $\mc{F}$ defined by $\rho$. 

\subsection{Rashomon set contains an $\varepsilon$-net for the hypothesis family}
As a first step, we use the optimization and generalization results of \cite{arora2019finegrained}
to show that, for any $\gamma>0$, the (empirical) Rashomon set $\mathcal{R}_{\text{set}}(\mc{F},\gamma)$ contains an $\varepsilon$-cover encompassing almost all of the volume of $\mc{F}$, 
provided that the noise level on the initial weights is small enough and the number of nodes in the network is large enough. The size of $\varepsilon$ depends on the complexity of the dataset as measured by the projection of the data labels on the eigenvectors of the Gram matrix: $\bs{y}^T (\bs{H}^{\infty})^{-1}\bs{y}$.

\begin{lemma}
\label{epsilon_covering_h}
        Let us assume that for some small $\bar{\epsilon},\delta>0$, the coefficient $\kappa$ satisfies \mbox{$\kappa = O\left(\frac{\bar{\epsilon}\delta}{n}\right)$}
        and the number of nodes $m$ in the hidden layer satisfies $m =\Omega\left(\frac{n^7}{\lambda_0^4\kappa^2\delta^4\bar{\epsilon}^2}\right)$.  
    Then there exists a subset $\mathcal{F}_{\delta} \subset \mathcal{F}$
    with volume at least $1-\delta$, such that for all $\gamma> \bar{\epsilon}^2/n$ the empirical (anchored) Rashomon set ${\hat{\mc{R}}}_{\text{set}}^{\text{anc}}(\mc{F}, \gamma)$ contains an $\varepsilon$-cover of $\mathcal{F}_{\delta}$. The value of $\varepsilon$ does not depend on $\gamma$ and can be expressed as:
    \begin{align} \label{epsilon_formula}
        \varepsilon = \sqrt{\bs{y}^T(\bs{H}^{\infty})^{-1}\bs{y}} +O\left(\frac{n\kappa}{\lambda_0\delta}\right)+\frac{poly(n,\lambda_0^{-1},\delta^{-1})}{n^{1/4}\kappa^{1/2}},
    \end{align}
    where $\sqrt{\bs{y}^T(\bs{H}^{\infty})^{-1}\bs{y}}$ is the dominant term.  
\end{lemma}

\begin{proof}
 Let $\gamma > \bar{\epsilon}^2/n$. Let us fix $\bs{a}=\bs{a}_0 \in \{-1,1\}^{\times m}$ and consider neural networks $\mathcal{F}^{\bs{a}}$ whose second layer has weights equal to $\bs{a}_0$. From Theorem 4.1 in \cite{arora2019finegrained}, we know the following: if $\kappa=O\left(\frac{\bar{\varepsilon}\delta}{\sqrt{n}}\right)$, $m=\Omega\left(\frac{n^7}{\lambda_0^4\kappa^2\delta^4\bar{\varepsilon}^2}\right)$ and $\eta=O\left(\frac{\lambda_0}{n^2}\right)$ then the empirical error of the two-layer network under a mean squared error loss function can be expressed as
    \begin{align}\label{emp_error_h}
    \text{empir.error}(k) = \frac{1}{n}\sum_{i=1}^{n}\left(y_i-f_{\bs{W}(k),\bs{a}}(x_i)\right)^2 = \frac{1}{n}\left(\sqrt{\sum_{i=1}^{n}\left(1-\eta\lambda_i\right)^{2k}\left(\bs{v}_i\bs{y}\right)^2} \pm \bar{\epsilon}\right)^2,
    \end{align}
    where $\lambda_i$ and $\bs{v}_i$ are the eigenvalues and corresponding eigenvectors of $\bs{H}^{\infty}$. 
  Since $\bs{H}^{\infty}$ is assumed to be positive-definite,
    according to \cite{du2019gradient} $\sum_{i=1}^{n}\left(1-\eta\lambda_i\right)^{2k}\left(\bs{v}_i\bs{y}\right)^2$ converges to $0$, therefore there exists $k_0$, such that the empirical error after $k_0$ iterations is less than $\gamma$.

    Lemma 5.3 in \cite{arora2019finegrained} states that the neurons of the network do not move much overall: if the width of the inner layer $m$ is sufficiently large ($m \geq \kappa^{-2}\text{poly}(n,\lambda_0^{-1}, \delta^{-1})$) and the step size satisfies $\eta = O\left(\frac{\lambda_0}{n^2}\right)$, then with probability at least $1-\delta$ over the random initialization we have for all numbers of iterations $k \geq 0$:
    \begin{align} \label{total_movement_neurons}
    \|\bs{W}(k)-\bs{W}(0)\|_F \leq \sqrt{\bs{y}^T(\bs{H}^{\infty})^{-1}\bs{y}} +O\left(\frac{n\kappa}{\lambda_0\delta}\right)+\frac{poly(n,\lambda_0^{-1},\delta^{-1})}{n^{1/4}\kappa^{1/2}}.
    \end{align}

    Therefore there exists a subset $\mc{F}_{\delta}^{\bs{a}}\subset \mc{F}^{\bs{a}}$, with volume at least $1-\delta$,
    such that for each function $f\in \mc{F}_{\delta}^{\bs{a}}$ 
    represented as  $\bs{W}(0)$,   
    there exists a number of iterations $k_0$ and a set of parameters $\bs{W}(k_0)$ for which at the same time 
    \begin{itemize}
        \item[(i)] $\|\bs{W}(0)-\bs{W}(k_0)\|_F \leq \varepsilon$
        \item[(ii)] and $f_{\bs{W}(k_0), \bs{a}} \in \hat{\mc{R}}_{\text{set}}^{\text{anc}}(\mc{F}, \gamma)$.
    \end{itemize}
   Therefore 
   \begin{align*}
P\left(dist(f_{\bs{W},\bs{a}},\hat{\mc{R}}_{\text{set}}^{\text{anc}})>\varepsilon\mid \bs{a}=\bs{a}_0\right) \leq \delta
   \end{align*}
   and 
   \begin{align*}
P\left(dist(f_{\bs{W},\bs{a}},\hat{\mc{R}}_{\text{set}}^{\text{anc}})>\varepsilon, \; \bs{a}=\bs{a}_0\right) = P\left(dist(f_{\bs{W},\bs{a}},\hat{\mc{R}}_{\text{set}}^{\text{anc}})>\varepsilon\mid \bs{a}=\bs{a}_0\right)P\left(\bs{a}=\bs{a}_0\right) \leq \delta \frac{1}{2^m}
   \end{align*}
   and taking the union over all possible values for $\bs{a}$ we get 
   \begin{align*}    P\left(dist(f_{\bs{W},\bs{a}},\hat{\mc{R}}_{\text{set}}^{\text{anc}})>\varepsilon\right) = \sum_{i=1}^{2^m}P\left(dist(f_{\bs{W},\bs{a}},\hat{\mc{R}}_{\text{set}}^{\text{anc}})>\varepsilon, \; \bs{a}=\bs{a}_i\right) \leq 2^m \delta \frac{1}{2^m}=\delta. 
   \end{align*}
Therefore taking the union $\mc{F}_{\delta}=\bigcup_{a\in \{ 1,-1\}^m} \mc{F}_{\delta}^{\bs{a}}$ leads to an $\mc{F}_{\delta}$ inside $\mc{F}$ with volume at least $1-\delta$ 
and the required properties.
\end{proof}

\subsection{Lower bound for the Rashomon ratio}
We are interested in the lower bound of the volume of the (empirical) Rashomon set. In order to use Lemma \ref{epsilon_covering_h} to bound the volume of the Rashomon set we need to explore its continuity property.

\begin{lemma} \label{rashomon_ball_h}
    For a family of classifiers $\mathcal{F}$ given by two-layer neural networks we have: 
    \begin{align*}
        \text{if  } f_{\bs{W^1},\bs{a}}\in \hat{\mathcal{R}}_{\text{set}}(\gamma/2, \mathcal{F}) \;\; \text{then for all } \bs{W^2} \in B\left(\bs{W^1}, \frac{\gamma}{2}
        \right):\;f_{\bs{W^2,\bs{a}}}\in \hat{\mathcal{R}}_{\text{set}}(\gamma, \mathcal{F}). 
    \end{align*}
\end{lemma}
\begin{proof}
   Any function in $ \mathcal{F}$ can be written in terms of its parameters $\bs{W}$ and $\bs{a}$ as
    \begin{align*}
        f_{\bs{W,a}}(\bs{x}) = \frac{1}{\sqrt{m}}\sum_{i=1}^ma_i\sigma(\bs{w}_i^T\bs{x}).
    \end{align*}
    We now show that if the $\bs{W}$ parameters of two such functions with the same  $\bs{a}$ parameter are close, then the values of the functions are close as well. This way, we will be able to guarantee that any function in the Rashomon set corresponding to a certain error $\gamma$ is surrounded by a ball of functions that lie in Rashomon set corresponding to a slightly bigger $\tilde{\gamma}$.
    
    Let $\bs{W^1,a}$ such that $f_{\bs{W^1},\bs{a}}\in \hat{\mathcal{R}}_{\text{set}}(\gamma/2, \mathcal{F})$. 
    Consider any $\bs{W^2}$ with $\|\bs{W^1}-\bs{W^2}\|_F\leq \omega$. For all of the rows $i$ ,we have $\|\bs{w_i^1}-\bs{w_i^2}\| \leq \omega$. By the triangle inequality 
    \begin{align*}
         \mid f_{\bs{W^1,a}}(\bs{x})- f_{\bs{W^2,a}}(\bs{x})\mid  \leq \frac{1}{\sqrt{m}}\sum_{i=1}^m\mid a_i\mid \mid \sigma(\bs{w_i^1}^T\bs{x})-\sigma(\bs{w_i^2}^T\bs{x})\mid 
    \end{align*}
    There are three cases for $\mid \sigma(\bs{w_i^1}^T\bs{x})-\sigma(\bs{w_i^2}^T\bs{x})\mid $:
    \begin{itemize}
        \item both dot products are positive: $\bs{w_i^1}^T\bs{x}>0$ and $\bs{w_i^2}^T\bs{x}>0$ then 
        \begin{align*}
        \mid \sigma(\bs{w_i^1}^T\bs{x})-\sigma(\bs{w_i^2}^T\bs{x})\mid  = \mid \bs{w_i^1}^T\bs{x}-\bs{w_i^2}^T\bs{x}\mid  \end{align*}
        \item one dot product is positive, one negative, say $\bs{w_i^1}^T\bs{x}>0$ and $\bs{w_i^2}^T\bs{x} \leq 0$, then 
        \begin{align*}
        \mid \sigma(\bs{w_i^1}^T\bs{x})-\sigma(\bs{w_i^2}^T\bs{x})\mid  = \mid \bs{w_i^1}^T\bs{x}-0\mid < \mid \bs{w_i^1}^T\bs{x}-\bs{w_i^2}^T\bs{x}\mid \end{align*}
        \item both dot products are negative: $\bs{w_i^1}^T\bs{x} \leq 0$ and $\bs{w_i^2}^T\bs{x} \leq 0$,
        then 
        \begin{align*}
        \mid \sigma(\bs{w_i^1}^T\bs{x})-\sigma(\bs{w_i^2}^T\bs{x})\mid  = 0 < \mid \bs{w_i^1}^T\bs{x}-\bs{w_i^2}^T\bs{x}\mid.  
        \end{align*}
    \end{itemize}
    In all these cases $\mid \sigma(\bs{w_i^1}^T\bs{x})-\sigma(\bs{w_i^2}^T\bs{x})\mid  \leq \mid \bs{w_i^1}^T\bs{x}-\bs{w_i^2}^T\bs{x}\mid $. Using the Cauchy-Schwartz inequality we get $\mid \bs{w_i^1}^T\bs{x}-\bs{w_i^2}^T\bs{x}\mid  \leq \|\bs{w^1}_i-\bs{w^2}_i\| \|\bs{x}\|$, therefore 
  \begin{align*}
        \mid f_{\bs{W^1,a}}(\bs{x})- f_{\bs{W^2,a}}(\bs{x})\mid  \leq \frac{1}{\sqrt{m}}\sum_{i=1}^m\mid a_i\mid \|\bs{w^1}_i-\bs{w^2}_i\| \|\bs{x}\|
    \end{align*}
    using the assumption that $\|\bs{x}\|=1$ and $\mid a_i\mid =1$ we get 
    \begin{align*}
        \mid f_{\bs{W^1,a}}(\bs{x})- f_{\bs{W^2,a}}(\bs{x})\mid  \leq \frac{1}{\sqrt{m}}\sum_{i=1}^m\|\bs{w^1}_i-\bs{w^2}_i\|,
    \end{align*}
    by Jensen's inequality (for real convex function $\phi$ it holds that $\phi(\frac{1}{n}\sum_{i=1}^{n}X_i)\leq \frac{1}{n}\sum_{i=1}^{n}\phi(x)$ applied to a convex function $\phi(x)=x^2$ leading to $\frac{1}{n}\sum_{i=1}^n X_i \leq \sqrt{\frac{1}{n}\sum_{i=1}^{n}X_i^2}$) we get:
       \begin{align*}
        \frac{1}{\sqrt{m}}\sum_{i=1}^m\|\bs{w^1}_i-\bs{w^2}_i\| \leq \sqrt{m}\sqrt{\frac{1}{m}\sum_{i=1}^m\|\bs{w^1}_i-\bs{w^2}_i\|^2} = \|\bs{W}^{1}-\bs{W}^{2}\|_F,
    \end{align*}
    which means
    \begin{align*}
        \mid f_{\bs{W^1,a}}(\bs{x})- f_{\bs{W^2,a}}(\bs{x})\mid  \leq \|\bs{W^1}-\bs{W}^2\|_F.
    \end{align*}
    We have that $f_{\bs{W^1,a}}(\bs{x}) \in \mathcal{R}_{\text{set}}(\gamma/2,\mathcal{F})$, which means that the empirical loss of $f_{\bs{W^1,a}}(\bs{x})$ is less than $\gamma/2$:
    \begin{align*}
        \frac{1}{n}\sum_{i=1}^n l\left(f_{\bs{W^1,a}}(\bs{x}_i),y_i\right) \leq \gamma/2
    \end{align*}
    Let us compare the empirical losses of the functions $f_{\bs{W^1,a}}(\bs{x})$ and $f_{\bs{W^2,a}}(\bs{x})$:
    \begin{align*}
       \mid \frac{1}{n}\sum_{i=1}^n l\left(f_{\bs{W^1,a}}(\bs{x}_i),y_i\right) - \frac{1}{n}\sum_{i=1}^n l\left(f_{\bs{W^2,a}}(\bs{x}_i),y_i\right)\mid  = \mid  \frac{1}{n}\sum_{i=1}^n l\left(f_{\bs{W^1,a}}(\bs{x}_i),y_i\right)- l\left(f_{\bs{W^2,a}}(\bs{x}_i),y_i\right)\mid  \\
        \leq  \frac{1}{n}\sum_{i=1}^n \mid  l\left(f_{\bs{W^1,a}}(\bs{x}_i),y_i\right)- l\left(f_{\bs{W^2,a}}(\bs{x}_i),y_i\right) \mid
    \end{align*}
    and since we assume that the loss function is 1-Lipschitz in the first variable we get
\begin{align*}
        \mid \frac{1}{n}\sum_{i=1}^n l\left(f_{\bs{W^1,a}}(\bs{x}_i),y_i\right) - \frac{1}{n}\sum_{i=1}^n l\left(f_{\bs{W^2,a}}(\bs{x}_i),y_i\right)\mid  \leq \\
        \leq \frac{1}{n}\sum_{i=1}^n \mid (f_{\bs{W^1,a}}(\bs{x}_i) -f_{\bs{W^2,a}}(\bs{x}_i)\mid  \leq \|\bs{W^1}-\bs{W^2}\|_F.
    \end{align*}
    If we choose $\bs{W^2}$ such that $\|\bs{W^1}-\bs{W^2}\|_F < \gamma/2$ (meaning in a ball of radius $\gamma/2$ around $\bs{W^1}$) we get that the difference between empirical errors of neural network with weights $\bs{W^1}$ and neural network with weights $\bs{W^2}$ is less than $\gamma/2$, therefore
    \begin{align*}
        \mid\frac{1}{n}\sum_{i=1}^n l\left(f_{\bs{W^2,a}}(\bs{x}_i),y_i\right)\mid \leq \gamma.
    \end{align*}
    As a consequence if $ \bs{W^2} \in B(\bs{W^1},\gamma/2)$ then $f_{\bs{W^2},\bs{a}} \in \hat{\mathcal{R}}_{\text{set}}(\gamma, \mathcal{F})$. 
\end{proof}

\begin{lemma}\label{rashomon_ratio_h}
    Let us consider the same assumptions as in Lemma \ref{epsilon_covering_h} and let us also assume that for every $\bs{a}$, $\mathcal{F}_{\delta}^{\bs{a}} $ from Lemma \ref{epsilon_covering_h} includes a ball of volume $1-3\delta/2$ (w.r.t. $\rho$) around the origin. Additionally, let us assume that $\varepsilon>\gamma/2$. Then the empirical (anchored) Rashomon ratio $\hat{\mathcal{R}}_{\text{ratio}}^{\text{anc}}(\gamma, \mathcal{F})$ with respect to $\rho$ for the two-layered neural network (with $m$ hidden nodes) applied on a dataset with Gram matrix $\bs{H}^{\infty}$ and labels $\bs{y}$ is at least 
    \begin{align} \label{lower_bound_formula}
        \hat{\mathcal{R}}_{\text{ratio}}(\gamma, \mathcal{F}) \geq \dfrac{(1-\delta)2^{\frac{d \times m}{2}}}{\kappa \Gamma\left(\frac{d\times m}{2}+1\right)}\left(\frac{\gamma}{2}\right)^{\frac{d \times m}{2}}\exp\left(-\frac{\varepsilon^2}{2\kappa^2}\right)\left[ \frac{1}{2}\left(\theta_3\left(0,\exp\left(-\frac{9\varepsilon^2}{2\kappa^2}\right)\right)+1\right)\right]^{d \times m},
    \end{align}
    where $\theta_3(z,q)$ is a Jacobi Theta Function defined as $\theta_3(z,q) = \sum_{n=-\infty}^{\infty} q^{n^2} e^{2niz}$
    and $\varepsilon$ is again
    \begin{align*}
        \varepsilon = \sqrt{\bs{y}^T(\bs{H}^{\infty})^{-1}\bs{y}} +O\left(\frac{n\kappa}{\lambda_0\delta}\right)+\frac{poly(n,\lambda_0^{-1},\delta^{-1})}{n^{1/4}\kappa^{1/2}}.
    \end{align*}
\end{lemma}

\begin{proof}
Let us first fix the value of ${\bs a}$ and first look for a lower bound for the Rashomon ratio for each $\mc{F}^{\bs{a}}$. Due to Lemma \ref{epsilon_covering_h} there exists a subset $\mc{F}^{\bs{a}}_{\delta}$ of $\mc{F}^{\bs{a}}$ with volume no less than $1-\delta$ for which functions in the Rashomon set form an $\varepsilon$-cover. We are assuming that $\mc{F}^{\bs{a}}_{\delta}$ contains a ball around the origin of a volume no less than $1-3\delta/2$. Let us call this ball $\mc{B}$; observe that the Rashomon set forms an $\varepsilon$-cover of $\mc{B}$ because it is a subset of $\mc{F}^{\bs{a}}_{\delta}$. 

Let us consider a regular grid of points in $d \times m$ dimensional space with step size $3\varepsilon$ in each dimension. The grid is a set of points 
\begin{align*}
G = \{\bs{p}_{i_1,i_2,...,i_{m\times d}} = [3\varepsilon i_1,3\varepsilon i_2, ..., 3\varepsilon i_{m\times d}], i_1,i_2,...,i_{m\times d}\in \mathbb{Z}\}.
\end{align*}
 
By Lemma \ref{epsilon_covering_h} there exists a function that belongs to the Rashomon set nearby each point in the grid that belong to $\mc{B}$ as well: $\forall \bs{p} \in G \cap \mc{B}$, $\exists f_{\bs{W}, \bs{a}} \in \hat{\mathcal{R}}_{\text{set}}(\gamma/2, \mathcal{F})$, such that $\|\bs{W}-\bs{p}\|_F < \varepsilon$. Due to the triangle inequality, the Rashomon points corresponding to different grid points are different.

By Lemma \ref{rashomon_ball_h} there is a ball of radius $\gamma/2$ around each Rashomon point that consists of functions from $\hat{\mathcal{R}}_{\text{set}}(\gamma, \mathcal{F})$. Let us denote by $B_{i_1,i_2,...i_{d \times m}}$ 
the ball that corresponds to point $\bs{p}_{i_{1},i_2,...i_{d \times m}}$ in the grid - its radius is $\gamma/2$. 

The volume of these balls (in the sense of $\rho$) depends on how far away they are from the origin, that is to say from the mean of the Gaussian used to randomly draw the functions in $\mc{F}$. 
For example, a ball of a given radius in the Euclidean sense will have a smaller volume if it is further away from the mean of the Gaussian. In order to obtain a lower bound on the Rashomon ratio, we assume that the 
balls are in the worst position - furthest away from the mean of the Gaussian. In each ball $B_{i_1,i_2,...,i_{d \times m}}$ 
we replace the values of $\rho_a(f)$ inside the ball by a constant equal to it's minimum value inside the ball. This way, the true volume of each ball will not be less than our estimate.

If we were not restricted by the set $\mc{B}$ and assumed that each point in the grid had its "Rashomon neighbor", we could lower-bound the volume of all of these "Rashomon balls" in space by:
    \begin{align} \label{lower_bound_ratio_h}
         \sum_{i_1=-\infty}^{\infty}\sum_{i_2=-\infty}^{\infty}...\sum_{i_{d \times m}=-\infty}^{\infty}  \frac{1}{(2\pi)^{\frac{d \times m}{2}}}\frac{1}{\kappa}e^{-\text{dist}^2_{i_1...i_{d\times m}}/(2\kappa^2)} \text{vol}_{d \times m}\left(\gamma/2\right),
    \end{align}
    where 
    \begin{align*}
        \text{dist}_{i_1,...i_{d \times m}} = \sqrt{\sum_{j=1}^{d \times m}\left(3\varepsilon i_j\right)^2}+\varepsilon
    \end{align*}
is the largest possible distance from the origin to a point in a Rashomon ball that belongs to the grid point $\bs{p}$ (note that these balls are not intersecting), and $\text{vol}\left(B_{d \times m}\left(\gamma/2\right)\right)$ is the volume of a ball in $\mathbb{R}^{d\times m}$ with radius $\gamma/2$. 

We have to be careful, because the points of the grid that do not belong to the set $\mc{B}$ may not be covered by the Rashomon set. Since the "Rashomon balls" are distributed evenly in space and a normal distribution is rotationally invariant, if we subtract all of the Rashomon balls that belong to some point on a grid outside of $\mc{B}$ from (\ref{lower_bound_ratio_h}), we subtract $3\delta/2$ of its volume. Therefore, since the grid is symmetric, we get a lower bound for the Rashomon ratio of $\mc{F}^{\bs{a}}$: 
    \begin{align*}
        \hat{\mathcal{R}}_{\text{ratio}}(\gamma, \mathcal{F}^{\bs{a}}) \geq (1-3\delta/2)2^{d \times m}\sum_{i_1=0}^{\infty}\sum_{i_2=0}^{\infty}...\sum_{i_{d \times m}=0}^{\infty}  \frac{1}{(2\pi)^{\frac{d \times m}{2}}}\frac{1}{\kappa}e^{-\text{dist}^2_{i_1...i_{d\times m}}/(2\kappa^2)} \text{vol}_{d \times m}\left(\gamma/2\right).
    \end{align*}
    Taking all terms that do not depend on the indexes $i_1,...i_{d \times m}$ out of the sum we get
    \begin{align*}
         \hat{\mathcal{R}}_{\text{ratio}}(\gamma, \mathcal{F}^{\bs{a}}) \geq (1-3\delta/2)2^{d \times m}\frac{1}{(2\pi)^{\frac{d \times m}{2}}}\frac{1}{\kappa} \text{vol}_{d \times m}\left(\gamma/2\right)\sum_{i_1=0}^{\infty}\sum_{i_2=0}^{\infty}...\sum_{i_{d \times m}=0}^{\infty}  e^{-\text{dist}^2_{i_1...i_{d\times m}}/(2\kappa^2)}.
    \end{align*}
    Using the formula for $d \times m$ dimensional ball we get the following expression for a lower bound of the Rashomon ratio:
    \begin{align*}
        \hat{\mathcal{R}}_{\text{ratio}}(\gamma, \mathcal{F}^{\bs{a}}) \geq (1-3\delta/2)\dfrac{2^{\frac{d \times m}{2}}}{\kappa \Gamma\left(\frac{d\times m}{2}+1\right)}\left(\frac{\gamma}{2}\right)^{\frac{d \times m}{2}}\sum_{i_1=0}^{\infty}\sum_{i_2=0}^{\infty}...\sum_{i_{d \times m}=0}^{\infty} e^{-\text{dist}^2_{i_1...i_{d\times m}}/(2\kappa^2)}.
    \end{align*}
    Let us now focus on lower bounding the sum:
    \begin{align*}
\sum_{i_1=0}^{\infty}\sum_{i_2=0}^{\infty}...\sum_{i_{d \times m}=0}^{\infty} e^{-\text{dist}^2_{i_1...i_{d\times m}}/(2\kappa^2)}.
    \end{align*}
    We have that
    \begin{align*}
       \text{dist}^2_{i_1,...i_{d \times m}} = \left(\sqrt{\sum_{j=1}^{d \times m}\left(3\varepsilon i_j\right)^2}+\varepsilon\right)^2 \geq {\sum_{j=1}^{d \times m}\left(3\varepsilon i_j\right)^2}+\varepsilon^2
    \end{align*}
    therefore 
    \begin{align*}
        e^{-\text{dist}^2_{i_1,...i_{d \times m}}} \geq e^{-{\sum_{j=1}^{d \times m}\left(3\varepsilon i_j\right)^2}+\varepsilon^2}
    \end{align*}
    hence we have
    \begin{align*}        \sum_{i_1=0}^{\infty}\sum_{i_2=0}^{\infty}...\sum_{i_{d \times m}=0}^{\infty} e^{-\text{dist}^2_{i_1...i_{d\times m}}/(2\kappa^2)} \geq e^{-\varepsilon^2/(2\kappa^2)}\prod_{j=1}^{d \times m}\sum_{i_j=0}^{\infty}e^{-(3\varepsilon i_j)^2/(2\kappa^2)}.
    \end{align*}
    and 
    \begin{align*}
        e^{-\varepsilon^2/(2\kappa^2)}\prod_{j=1}^{d \times m}\sum_{i_j=0}^{\infty}e^{-(3\varepsilon i_j)^2/(2\kappa^2)} = \exp\left(-\frac{\varepsilon^2}{2\kappa^2}\right)\left[ \frac{1}{2}\left(\theta_3\left(0,\exp\left(-\frac{9\varepsilon^2}{2\kappa^2}\right)\right)+1\right)\right]^{d \times m},
    \end{align*}
    where $\theta_2(z,q)$ is a Jacobi Theta Function defined as 
    \begin{align*}
        \theta_3(z,q) = \sum_{n=-\infty}^{\infty} q^{n^2} e^{2niz}.
    \end{align*}
    The resulting lower bound for the Rashomon ratio is therefore
    \begin{align}\label{final_result_before_delta}\begin{split}
        &\hat{\mathcal{R}}_{\text{ratio}}(\gamma, \mathcal{F}^{\bs{a}}) \geq\\
        &\geq(1-3\delta/2)\dfrac{2^{\frac{d \times m}{2}}}{\kappa \Gamma\left(\frac{d\times m}{2}+1\right)}\left(\frac{\gamma}{2}\right)^{\frac{d \times m}{2}}\exp\left(-\frac{\varepsilon^2}{2\kappa^2}\right)\left[ \frac{1}{2}\left(\theta_3\left(0,\exp\left(-\frac{9\varepsilon^2}{2\kappa^2}\right)\right)+1\right)\right]^{d \times m}.
    \end{split}
    \end{align}
    What about the Rashomon ratio of the entire set $\mc{F}$? We know that
    \begin{align}\label{cond_a_ratio}
        \hat{\mc{R}}^{\text{anc}}_{\text{ratio}}(\mc{F}, \gamma) = P\left(f_{\bs{W},\bs{a}}\in \hat{\mc{R}}_{\text{set}}^{\text{anc}}(\mc{F}, \gamma)\right)= \sum_{i=1}^{2^m} P\left(f_{\bs{W},\bs{a}}\in \mc{R}_{\text{set}}(\mc{F}^{\bs{a}_i}, \gamma),\; \bs{a}=\bs{a}_i\right),
        \end{align}
        which is equal to
        \begin{align*}
        \sum_{i=1}^{2^m} P\left(f_{\bs{W},\bs{a}}\in \mc{R}_{\text{set}}(\mc{F}^{\bs{a}_i}, \gamma)\mid \; \bs{a}=\bs{a}_i\right)P\left(\bs{a}=\bs{a}_i\right)=\frac{1}{2^m}\sum_{i}^{2^m}P\left(f_{\bs{W},\bs{a}}\in \mc{R}_{\text{set}}(\mc{F}^{\bs{a}_i}, \gamma)\mid \; \bs{a}=\bs{a}_i\right).
    \end{align*}
    We computed the lower bound for $P\left(f_{\bs{W},\bs{a}}\in \mc{R}_{\text{set}}(\mc{F}^{\bs{a}_i}, \gamma)\mid \; \bs{a}=\bs{a}_i\right)$ in \ref{final_result_before_delta} which is the same for every vector $\bs{a}_i$, therefore due to \ref{cond_a_ratio} the same bound holds for an overall Rashomon ratio:
    \begin{align*}
        &\hat{\mc{R}}_{\text{ratio}}^{\text{anc}}(\mc{F}, \gamma) \geq\\
        &\geq(1-3\delta/2)\dfrac{2^{\frac{d \times m}{2}}}{\kappa \Gamma\left(\frac{d\times m}{2}+1\right)}\left(\frac{\gamma}{2}\right)^{\frac{d \times m}{2}}\exp\left(-\frac{\varepsilon^2}{2\kappa^2}\right)\left[ \frac{1}{2}\left(\theta_3\left(0,\exp\left(-\frac{9\varepsilon^2}{2\kappa^2}\right)\right)+1\right)\right]^{d \times m}.
    \end{align*}
\end{proof}

As an illustration, let us consider classifying a toy dataset with a two-layered neural network and construct the lower bound of the empirical Rashomon ratio according to formula \ref{lower_bound_formula}. We took the Iris dataset restricted to the first two classes (Setosa, Versicolor). For this dataset we computed its Gram matrix $\bs{H}^{\infty}$ (after normalizing all of the entries so that $\|\bs{x}\|=1$) and an approximation of $\varepsilon$ (based on the dominant term $\sqrt{\bs{y}^T(\bs{H}^{\infty})^{-1}\bs{y}}$ that is related to the complexity of the dataset). We assumed that a neural network with $m=4$ nodes in its hidden layer was applied to classify this dataset. In Figure \ref{lower_bound_kappa} we plotted the curves that illustrate how Formula (\ref{lower_bound_formula}) changes with parameter $\kappa$ that stands for the magnitude of the noise at the initialization of neural network's parameters $\bs{W}$. We plot three curves, each one corresponds to a different level of $\gamma$ in the Rashomon ratio $\mc{R}_{\text{ratio}}(\gamma, \mc{F})$. Blue dotted curve represents $\gamma=0.10$, orange curve with triangles stands for $\gamma=0.11$ and green curve with pentagons is for $\gamma=0.12$. As expected, the larger the $\gamma$, the larger is the value for the Rashomon ratio. The order of the lower bound for the Rashomon ratio measured here is $10^{-10}$, which is quite large when we consider the results from \cite{semenova_rudin}. The large Rashomon ratios from \cite{semenova_rudin} were of order $10^{-39}$ or $10^{-40}$, while the small Rashomon ratios were of order $10^{-42}$ or less. The setup of their measurements is very different though: they consider real datasets, while we have only a toy example and their family of classifiers covers different methods such as logistic regression, CART or random forests (estimated by an approximating set of decision trees of bounded depth), while ours is restricted to a small neural network where we vary the parameters. Also our $\gamma$ is set to $0.1$, while theirs is $0.05$. On the other hand we provide a lower bound for the Rashomon set in Figure \ref{lower_bound_kappa}, while in \cite{semenova_rudin} authors give an estimate for the Rashomon ratio itself. In all three curves we can see a growing trend in the Rashomon ratio with respect to $\kappa$. In \cite{semenova2023a} the authors show that more noise in the data lead to a bigger Rashomon ratio (for example for ridge regression). The results in Figure \ref{lower_bound_kappa} suggest that adding noise to the probability distribution on the parameter family $\bs{W}$ (in our case it corresponds to a larger $\kappa$) may also lead to a larger Rashomon ratio.  

 \begin{figure}[h]
 \centering
    \includegraphics[width=8cm]{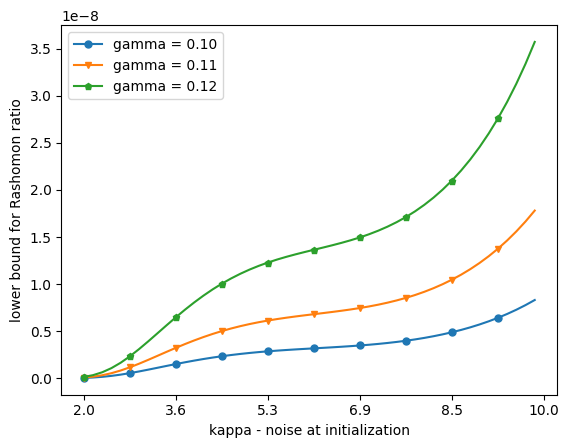}
    \caption{Lower bound for the Rashomon ratio that approximates the formula (\ref{lower_bound_formula}). Parameter $\kappa$ that corresponds to the magnitude of the noise at initialization varies between $2$ and $10$. Parameter $\varepsilon$ is equal to $7.13$ - that is the value of the dominant term in formula (\ref{epsilon_formula}) for the Iris dataset (where only two first classes are taken into consideration). The dimensionality of data in this dataset is $d=4$ and we assume that the neural network used has $m=4$ nodes in its hidden layer. Probability of failure is set to $\delta=0.1$. We consider three different values for the parameter $\gamma$: blue line with dots corresponds to $\gamma=0.10$, orange line with triangles corresponds to $\gamma=0.11$ and green line with pentagons corresponds to $\gamma=0.12$.}
    \label{lower_bound_kappa}
    \end{figure}

\section{Advantage of a large Rashomon ratio for infinite classifier families}
\label{section:advantage}

Assuming that we have a set of classifiers $\mathcal{F}_2$ with a large Rashomon set $\mathcal{R}_{\text{set}}(\mathcal{F}_2,\gamma)$ we would like to explore what happens to the training and population loss when we restrict ourselves to a smaller set of classifiers $\mathcal{F}_1 \subset \mathcal{F}_2$. Specifically, we are interested in a set $\mathcal{F}_1$ chosen from $\mathcal{F}_2$ at random according to a distribution with probability measure $\rho$ on $\mathcal{F}_2$. To study this question we first propose an extension of Theorem 5 from \cite{semenova_rudin}
to the case where $\mathcal{F}_1$ and $\mathcal{F}_2$ are infinite sets of functions. This way we can apply these results, for example, to the set $\mathcal{F}_2$ of affine classifiers and a subset $\mathcal{F}_1$ of randomly chosen affine classifiers.  
Theorem \ref{theorem_5} together with Theorem \ref{theorem_6} guarantee that, given a large enough Rashomon ratio, the empirical loss of the random subfamily is close to the population loss of the entire family. At the same time, its generalization properties are better. 

We will need the following Lemma to prove Theorem \ref{theorem_5}. It is called Proposition 13 in \cite{semenova_rudin} and it shows that if a function belongs to the empirical anchored Rashomon set with large probability it also belongs to the true anchored Rashomon set:

\begin{lemma}\label{true_is_close_to_empirical}
    For a loss $l$ bounded by $b$ and for any $\varepsilon, \eta>0$, for a fixed $f\in \mc{F}$ we have
    \begin{align}
        P\left[f \in \hat{\mc{R}}_{\text{set}}^{\text{anc}}(\mc{F},\eta+\varepsilon) \mid f \in {\mc{R}}_{\text{set}}^{\text{anc}}(\mc{F},\eta) \right] \geq 1-\exp(-2n(\varepsilon/b)^2),
    \end{align}
    where probability is taken with respect to the random draw of training data. 
\end{lemma}
\begin{proof}
    By Hoeffding's inequality for a fixed $f \in \mc{R}_{\text{set}}^{\text{anc}}(\mc{F}, \eta)$ we get that 
    \begin{align*}
        P \left(L(f)-\hat{L}(f)>\varepsilon\right) \leq \exp(-2n(\varepsilon/b)^2). 
    \end{align*}
    Since  $f \in \mc{R}_{\text{set}}^{\text{anc}}(\mc{F}, \eta)$, we know that $L(f) \leq \eta$, therefore $\hat{L}(f) \leq \eta+\varepsilon$, which is equivalent to $f \in \hat{\mc{R}}_{\text{set}}^{\text{anc}}(\mc{F}, \eta+\varepsilon)$ with probability at least $1-\exp(-2n(\varepsilon/b)^2)$.
\end{proof}

\begin{theorem}\label{theorem_5}
    Let us consider two hypothesis spaces such that $\mc{F}_1 \subset \mc{F}_2$. Let us also consider a classification problem with a loss function $l$ bounded by $b$. Let us assume that the true Rashomon set of $\mc{F}_2$ intersects $\mathcal{F}_1$, i.e.~there exists a model $\tilde{f_1}\in \mathcal{F}_1$ such that $\tilde{f}_1 \in \mc{R}_{set}(\mathcal{F}_2, \gamma)$. Then for any $\delta>0$, with probability at least $1-\delta$ with respect to the random draw of the data, we have
    \begin{align}
        \inf_{f_2 \in \mc{F}_2}L(f_{2})-\Delta \leq \inf_{f_1 \in \mc{F}_1}\hat{L}({f}_1) \leq \inf_{f_2\in \mc{F}_2} L(f_2)+\gamma+b\sqrt{\frac{\ln1/\delta}{2n}},
    \end{align}
where $\Delta$ is either
\begin{itemize}
\item[(i)] 
\begin{align*}
    \Delta = b\sqrt{\frac{1}{n}\left(d_{VC}\left(\ln\frac{2n}{d_{VC}}+1\right)-\ln\left(\dfrac{\delta}{4}\right)\right)},
\end{align*} 
if the smaller set $\mathcal{F}_1$ has a finite VC dimension $d_{VC}=d_{VC}(\mathcal{F}_1)$ or
\item[(ii)] 
\begin{align}
    \Delta = \sqrt{\frac{8}{n}\ln\left(\frac{4m_{\mathcal{F}_1}(2n)}{\delta}\right)},
\end{align}
when the loss $l$ is the 0-1 loss  and the smaller set has a growth function $m_{\mathcal{F}_1}(n)$ (which can be substantially smaller than that of $\mathcal{F}_2$).
\end{itemize}
\end{theorem}

That means that we can bound the empirical loss within $\mathcal{F}_1$ with the population loss within $\mathcal{F}_2$ even when $\mathcal{F}_1$ and $\mathcal{F}_2$ are infinite,
so long as $\mathcal{F}_1$ has either a finite VC dimension or growth function expressible in a closed form. The proof differs from that of \cite{semenova_rudin} in that it replaces the union bound by either a bound using the growth function of $\mathcal{F}_1$ or its VC dimension.
\begin{proof}
Let us consider the lower and upper bounds separately:
\begin{itemize}
\item[(i)] (Lower bound) 
For the case of a 0-1 loss, it is known (e.g., see \cite{AbuMostafa2012LearningFD}) that,
with probability at least $1-\delta$,
    \begin{align*}
        \sup_{f_1 \in \mathcal{F}_1} \mid L(f_1)-\hat{L}(f_1)\mid  \leq \sqrt{\frac{8}{n}\ln\left(\frac{4m_{\mathcal{F}_1}(2n)}{\delta}\right)}.
    \end{align*}
    For other loss functions $l$, bounded by $b$, according to the case of bounded functions from \cite{Vapnik1998}, with probability at least $1-\delta$
    \begin{align*}
        \sup_{f_1 \in \mathcal{F}_1} \left(L(f_1) - \hat{L}(f_1)\right)\leq 
        \dfrac{b}{2}\sqrt{4\dfrac{d_{VC}\left(\ln\frac{2n}{d_{VC}}+1\right)-\ln\left(\dfrac{\delta}{4}\right)}{n}},
    \end{align*}
    Therefore 
    \begin{align*}
         \inf_{f_2 \in \mc{F}_2}L(f_{2}) \leq L(\hat{f}_1) \leq \inf_{f_1 \in \mc{F}_1}\hat{L}({f}_1) + b\sqrt{\frac{1}{n}\left(d_{VC}\left(\ln\frac{2n}{d_{VC}}+1\right)-\ln\left(\dfrac{\delta}{4}\right)\right)}
    \end{align*}
    or using a growth function for 0-1 loss:
    \begin{align*}
      \inf_{f_2 \in \mc{F}_2}L(f_{2}) \leq L(\hat{f}_1) \leq \inf_{f_1 \in \mc{F}_1}\hat{L}({f}_1) + \sqrt{\frac{8}{n}\ln\left(\frac{4m_{\mathcal{F}_1}(2n)}{\delta}\right)}. 
    \end{align*}
    \item[(ii)] (Upper bound) This part of the proof holds the same for infinite $\mathcal{F}_1$ and $\mathcal{F}_2$ as for their finite counterparts \cite{semenova_rudin}, except that the concept of minimum needs to be replaced by infimum. We assume that there exists $\tilde{f}_1\in\mathcal{F}_1 \cap \mathcal{R}_{\text{set}}(\mathcal{F}_2, \gamma)$, therefore $L(\tilde{f}_1) \leq  \inf_{f_2 \in \mc{F}_2}L(f_{2})+\gamma$. 
    Since $ \inf_{f_1 \in \mc{F}_1}L(f_{1}) \leq L(\tilde{f}_1)$, there exists a function $f_1^{\star}\in \mc{F}_1$ for which we have that $L(f_1^{\star}) \leq L(\tilde{f}_1)$, therefore $f_1^{\star}\in \mathcal{R}_{\text{set}}(\mathcal{F}_2, \gamma)$ and $f_1^{\star}\in \mathcal{R}_{\text{set}}^{\text{anc}}(\mathcal{F}_2, \eta)$, where $\eta =\inf_{f_2 \in \mc{F}_2}L(f_{2})+\gamma$. According to Lemma \ref{true_is_close_to_empirical} we have that for every $\delta_1 >0, \; f_1^{\star} \in \hat{\mathcal{R}}_{\text{set}}^{\text{anc}}(\mathcal{F}_2, \eta + \delta_1)$ with probability at least $1-e^{-2n(\delta_1/b)^2}$.  Therefore, with the same probability, $\hat{L}(f_1^{\star})\leq \inf_{f_2 \in \mc{F}_2}L(f_{2})+\gamma +\delta_1$. If we set $\delta = e^{-2n(\delta_1/b)^2}$ we get that with probability $1-\delta$ it holds that $\hat{L}(f_1^{\star})\leq \inf_{f_2 \in \mc{F}_2}L(f_{2})+\gamma+b\sqrt{\frac{\ln(1/\delta)}{2n}}$. By definition of 
    \begin{align*}
        \inf_{f_1 \in \mc{F}_1}\hat{L}({f}_1) \leq \hat{L}(f_{1}^{\star}) \leq \inf_{f_2 \in \mc{F}_2}L(f_{2})+\gamma+b\sqrt{\frac{\ln(1/\delta)}{2n}}.
    \end{align*}
\end{itemize}
Combining part (i) and part (ii), we get the following inequalities:
\begin{align*}
        \inf_{f_2 \in \mc{F}_2}L(f_{2})-b\sqrt{\frac{1}{n}\left(d_{VC}\left(\ln\frac{2n}{d_{VC}}+1\right)-\ln\left(\dfrac{\delta}{4}\right)\right)} \leq \inf_{f_2 \in \mc{F}_2}L(f_{2})+\gamma+b\sqrt{\frac{\ln1/\delta}{2n}},
\end{align*}
or
\begin{align*}
    \inf_{f_2 \in \mc{F}_2}L(f_{2}))- \sqrt{\frac{8}{n}\ln\left(\frac{4m_{\mathcal{F}_1}(2n)}{\delta}\right)} \leq \inf_{f_1 \in \mc{F}_1}\hat{L}({f}_1) \leq \inf_{f_2 \in \mc{F}_2}L(f_{2})+\gamma+b\sqrt{\frac{\ln1/\delta}{2n}},
\end{align*}
for a 0-1 loss $l$. 
\end{proof}

\begin{theorem}\label{theorem_6}[$\mc{F}_1$ as a random subset of $\mc{F}_2$ under large Rashomon ratio condition] Consider a finite hypothesis space $\mathcal{F}_1$ and a finite or infinite hypothesis space $\mathcal{F}_2$ equipped with a distribution characterized by a probability measure $\rho$, such that $\mathcal{F}_1 \subset \mathcal{F}_2$ and each function in $\mathcal{F}_1$ is drawn independently from $\mathcal{F}_2$ according to $\rho$. For loss $l$ bounded by $b$ and true Rashomon ratio with respect to $\rho$:
\begin{align*}
    \mathcal{R}_{\text{ratio}}(\mc{F}_2, \gamma) = \int_{\mathcal{F}_2}\mathbbm{1}(L(f)\leq \gamma)\rho(df)
\end{align*}
and for any $\delta>0$ we have that if Rashomon ratio is at least
\begin{align*}
    \mathcal{R}_{\text{ratio}}(\mathcal{F}_2, \gamma) \geq 1-\delta^{\frac{1}{\mid \mathcal{F}_1\mid }},
\end{align*}
with probability at least $(1-\delta)$ with respect to the random draw of functions from $\mathcal{F}_2$ to form $\mathcal{F}_1$, the assumptions of Theorem \ref{theorem_5} hold. 
\end{theorem}

\begin{proof}
The probability that the Rashomon set of $\mc{F}_2$ does not include any of the functions from $\mc{F}_1$ - an independent random draw of $\mid \mc{F}_1\mid $ functions according to $\rho$ equals
\begin{align*}
    (1-\mc{R}_{\text{ratio}})^{\mid \mc{F}_1\mid }.
\end{align*}
We want this probability to be less than $\delta>0$, which is equivalent with
\begin{align}\label{inequality}
\begin{split}
    (1-\mc{R}_{\text{ratio}})^{\mid \mc{F}_1\mid }\leq \delta\\
     \mc{R}_{\text{ratio}} \geq 1 -  \delta^{\frac{1}{\mid \mc{F}_1\mid }},
\end{split}
\end{align}
therefore if Rashomon ratio is large enough, with probability at least $1-\delta$, one of the functions from $\mathcal{F}_1$ will lay inside the the Rashomon set.  
\end{proof}

\begin{corollary}
Equivalently, we can rephrase Theorem \ref{theorem_6} as either (i) or (ii) below: 
    \begin{itemize}
    \item[(i)] 
    if the subset $\mathcal{F}_1$ is large enough: if 
    \begin{align*}
    \mid \mathcal{F}_1\mid  > \frac{\ln(\delta)}{\ln(1-\mathcal{R}_{\text{ratio}}(\mathcal{F}_2,\gamma))},
    \end{align*}
    then the assumptions of Theorem \ref{theorem_5} hold with probability at least $1-\delta$.
    \item[(ii)] If the probability that the assumptions of Theorem \ref{theorem_5} are satisfied is large when $\mid \mathcal{F}_1\mid $ is relatively small, the Rashomon ratio must be large, i.e.
    \begin{align*}
        \mathcal{R}_{\text{ratio}} \geq 1-\delta^{\frac{1}{\mid \mathcal{F}_1\mid }}
    \end{align*}
\end{itemize}
\end{corollary}
\begin{proof}
    Both (i) and (ii) can be shown using inequality \ref{inequality}.
\end{proof}

If the smaller family $\mc{F}_1$ is infinite we can still use Theorem \ref{theorem_6} as long as we can draw a subset of $\mc{F}_1$ of a certain size $N$ independently. We can apply this idea to the 
thresholding after random projection (TARP) from \cite{boutin_coupkova}.
In this method, a classifier is chosen at random by generating a finite number of projection directions. For each projection,  the threshold that yields the best classification is chosen. The best classifier among this finite number of thresholding after random projection is then selected. Thus, in this case, the restricted family of classifiers considered is not a random subset of the entire family, as the threshold is chosen by optimization rather than random. However, one can view  ${\mathcal F}_1$ as the infinite family consisting of the union of all possible thresholds for the finitely many randomly chosen projection directions. 
This results in an infinite set of classifiers. More precisely, 
the family ${\mathcal F}_2$ is the set of affine classifiers (potentially after extending the data to a higher-dimensional space). However, we can choose a subset of independent functions from it by choosing the threshold at random as well. Then we get $\mc{F}_1 \supset \bar{\mc{F}_1} = \{y = \bs{a}_i\bs{x}+t, \; \bs{a}_i \in \{\bs{a}_1,...,\bs{a}_N\},\; t\}$, where $\{\bs{a}_1,...,\bs{a}_N\}, t$ are fixed and randomly chosen and therefore $\mid \bar{\mc{F}}_1\mid  = N$. If we know that the intersection between the chosen models and the Rashomon set of all affine models is nonempty, we can obtain the following bounds for the method of thresholding after random projection:
\begin{align*}
      \inf_{f_2 \in \mc{F}_2}L(f_{2})- \sqrt{\frac{8}{n}\ln\left(\frac{16nN}{\delta}\right)} \leq \text{empirical loss of TARP} \leq \inf_{f_2 \in \mc{F}_2}L(f_{2})+\gamma+\sqrt{\frac{\ln1/\delta}{2n}}.
\end{align*}

\subsection{Application of the Theorems to the experimental results}
\begin{itemize}
    \item {Example from Section \ref{affine_gauss}}:
Consider a specific example from Figure \ref{one_dim_circle} part a) of affine classification of one-dimensional mixture of Gaussians, where $\sigma=1$, $2\mu=5$. According to our experiments approximately $58\%$ of the classifiers belong to the Rashomon set with $\gamma=0.05$. According to Lemma \ref{ratio_hoeffding} with probability at least $98\%$ the true Rashomon ratio does not differ from the estimated one by more than $0.05$ since we use $N=1000$ functions for the estimation. Therefore
\begin{align}\label{part_1_application}
    P\left(\mc{R}_{\text{ratio}}(\mc{F}_{\text{af}, 0.05})\geq 0.58-0.05 = 0.53\right) \geq 0.98
\end{align}
We can apply Theorem \ref{theorem_6} and see that with probability $99\%$ (given by $\delta=0.01$) if 

\begin{align}\label{part_2_application}
N = \mid\mathcal{F}_1\mid \geq \frac{\ln(\delta)}{\ln(1-\mc{R}_{ratio}(\mc{F}_2,\gamma))}
\end{align}

then in the smaller set of classifiers $\mathcal{F}_{1}$ there is at least one function that belongs to the Rashomon set of the larger family $\mathcal{F}_2$, which is all affine classifiers in this particular case.
We can combine \ref{part_1_application} and \ref{part_2_application} to get that with probability at least $0.98*0.99 = 0.97$ if we choose $N$ large enough:
\begin{align*}
    N \geq  \frac{\ln(0.01)}{\ln(1-0.53)} \approx 6.1 \geq \frac{\ln(\delta)}{\ln(1-\mc{R}_{ratio}(\mc{F}_2,\gamma))}
\end{align*}

Theorem \ref{theorem_5} applies to the smaller set $\mathcal{F}_1$ and we can achieve a good training error while having a small generalization gap (the left hand size is based on $\mid \mc{F}_1\mid=N$):
\begin{align}
     \inf_{f_2 \in \mc{F}_2}L(f_{2})- \sqrt{\frac{1}{2n}\ln\left(\frac{2N}{\delta}\right)} \leq \inf_{f_1 \in \mathcal{F}_1}\hat{L}(f_1) \leq \inf_{f_2 \in \mc{F}_2}L(f_{2})+\gamma+\sqrt{\frac{1}{2n}\ln\left(\frac{1}{\delta}\right)},
\end{align}
where $\inf_{f_2 \in \mc{F}_2}L(f_{2})$ is the Bayes error for the mixture of Gaussian distributions as in \ref{bayes_error_one_dim}.

\item{Example from Section \ref{twolayernn}}: we can not apply Theorem \ref{theorem_5} directly to the second case that we consider in this work (a two-layer neural network used to classify the Iris dataset), because in Section \ref{twolayernn} we lower bound the empirical anchored Rashomon ratio and not the true Rashomon ratio. We can still use Theorem \ref{theorem_6} to compute the lower bound on the probability that we pick at random a classifier that belongs to the empirical anchored Rashomon set. The lower bound for the empirical anchored Rashomon ratio for $\gamma=0.11$ for two-layered Neural network with $\kappa$ approximately equal to $5$ applied to the Iris dataset is close to $0.5\times 10^{-8}$ (see Figure \ref{lower_bound_kappa}), therefore according to Theorem \ref{theorem_6} with probability $0.9$ if we have 
 \begin{align*}
N > \mid\mathcal{F}_1\mid \geq \frac{\ln(\delta)}{\ln(1-\mc{R}_{ratio}(\mc{F}_2,\gamma))} = 4.6 \times 10^{8}
\end{align*}
then in the smaller set of classifiers $\mathcal{F}_{1}$ there is at least one function that belongs to the empirical anchored Rashomon set of the larger family $\mathcal{F}_2$. Therefore if we apply ERM on the smaller set $\mc{F}_1$ we have that with probability at least $0.9$:
\begin{align*}
    \inf_{f_1 \in \mc{F}_1} \hat{L}(f_1) \leq \gamma = 0.11.
\end{align*}
\end{itemize}

\section{Conclusion}
Semenova et al.~\cite{semenova_rudin} showed that, for the case of a finite family of classifiers, having a large Rashomon ratio is advantageous because it allows one to pick a classifier from a simpler/smaller subfamily of functions and expect that it will have a good accuracy.  We extended this result to the case of infinite families of classifiers so to be able to apply similar results to more general classifiers, including thresholding after random projection \cite{boutin_coupkova}, random projection ensemble classifiers \cite{cannings2017randomprojection}, and untrained neural networks such as \cite{baek2021face}. Whether considering a finite or infinite family of classifiers, we now have a theoretical basis for solving classification problems with a large Rashomon ratio by picking a classifier among a small set of randomly chosen ones, based on their performance on a training dataset. In general, restricting the search to a small set of classifiers is likely to decrease the generalization gap of the classifier. If the classification problem has a large Rashomon ratio (with respect to the large family), then the theory guarantees that the training accuracy will likely not be reduced too much when picking from the small set of classifiers. 

Knowing this, the question of being able to compute or estimate the Rashomon ratio for specific classification problems becomes even more relevant. We showed that a Monte-Carlo type method can be used to estimate the Rashomon ratio simply by counting randomly chosen functions whose empirical accuracy is above a certain value, and we quantified the error of this estimate. But whether counting samples or computing analytically, a key issue in obtaining the Rashomon ratio for an infinite family of classifiers is the seemingly infinite volume of the classifier family. For affine classifiers, this can be addressed by compactifying the parameter set of the classifier family by mapping it to a sphere. The Rashomon ratio can then be computed by putting a probability measure on the sphere (uniform or not).  More generally, a probability measure $\rho$ with $\int_{{\mathcal F}} \rho(df) =1 $ must be put on the classifier family ${\mathcal F}$, and all results depend on the choice of this measure. Choosing the measure with a specific classification problem in mind can increase the Rashomon ratio and guide the classifier selection method to improve the chance of success. 

We computed the Rashomon ratio for the problem of separating two normally distributed classes in $d$ dimension using an affine classifier. We showed that, by choosing $\rho (f)$ as a Uniform distribution over a sphere whose zero-angle point corresponds to the middle point between the two Gaussian means then the ratio is large and converges to $1$ as the distance between the means grows to infinity. Our results highlight the relationship between the choice of probability measure (indirectly also the map chosen to compactify the classifier family) and the Rashomon ratio.

We also studied the Rashomon ratio for a classification problem with a given Gram matrix $\bs{H}^\infty$ using a randomly initialized two-layer ReLU neural network. We obtained a lower bound on the Rashomon ratio in terms of the dimensionality of the dataset, the number of nodes in the hidden layer of the neural network, and $\bs{y}^T(\bs{H}^{\infty})^{-1}\bs{y}$, where $\bs{y}$ is a vector containing the labels of the points. Random labels yield a smaller lower bound than real labels, which is consistent with the observation that real classification problems tend to be easier to solve than made up ``random" problems. This is also consistent with the findings of \cite{arora2019finegrained}, which showed that the empirical loss of these ReLU neural networks is greater when the data labels are randomly chosen. 

For this former case, we also showed that the Rashomon set is an epsilon cover of the classifier family equipped with a Euclidean metric (on the parameters space). The dominant term of epsilon depends only on $\bs{y}$ and $\bs{H}^{\infty}$, and indicates that a random initialization of the network is expected to be close to the Rashomon set. This gives a geometric interpretation to the fact proved in \cite{arora2019finegrained} that a random initialization of the neural network is expected to be close to the true solution.  It is conceivable that, more generally, the good performance of neural networks can be explained geometrically by the large size of the Rashomon set. In particular, if a certain task can be solved successfully with an untrained neural network (e.g., visual number sense \cite{kim2021visual} or face detection \cite{baek2021face}), this may be due to the large size of the Rashomon set, which could be related to $\bs{y}^T(\bs{H}^{\infty})^{-1}\bs{y}$ being small.

In future research, it would be interesting to study the question of how to construct or choose a probability measure on the classifier family so to end up with a large Rashomon ratio for a given classification problem.  

\bibliographystyle{plain}
\bibliography{bibliography}

\begin{thebibliography}{10}

\bibitem{AbuMostafa2012LearningFD}
Yaser~S Abu-Mostafa, Malik Magdon-Ismail, and Hsuan-Tien Lin.
\newblock {\em Learning from data}, volume~4.
\newblock AMLBook New York, 2012.

\bibitem{arora2019finegrained}
Sanjeev Arora, Simon Du, Wei Hu, Zhiyuan Li, and Ruosong Wang.
\newblock Fine-grained analysis of optimization and generalization for overparameterized two-layer neural networks.
\newblock In {\em International Conference on Machine Learning}, pages 322--332. PMLR, 2019.

\bibitem{baek2021face}
Seungdae Baek, Min Song, Jaeson Jang, Gwangsu Kim, and Se-Bum Paik.
\newblock Face detection in untrained deep neural networks.
\newblock {\em Nature communications}, 12(1):7328, 2021.

\bibitem{boutin_coupkova}
Mireille Boutin and Evzenie Coupkova.
\newblock Optimality and complexity of classification by random projection, 2021.

\bibitem{breiman2003statistical}
Leo Breiman.
\newblock Statistical modeling: The two cultures.
\newblock {\em Quality control and applied statistics}, 48(1):81--82, 2003.

\bibitem{cannings2017randomprojection}
Timothy~I Cannings and Richard~J Samworth.
\newblock Random-projection ensemble classification.
\newblock {\em Journal of the Royal Statistical Society Series B: Statistical Methodology}, 79(4):959--1035, 2017.

\bibitem{mythesis}
Evzenie Coupkova.
\newblock {\em Random parameters in learning: advantages and guarantees}.
\newblock Phd thesis, Purdue University, West Lafayette, IN, May 2024.

\bibitem{d2022underspecification}
Alexander D'Amour, Katherine Heller, Dan Moldovan, Ben Adlam, Babak Alipanahi, Alex Beutel, Christina Chen, Jonathan Deaton, Jacob Eisenstein, Matthew~D Hoffman, et~al.
\newblock Underspecification presents challenges for credibility in modern machine learning.
\newblock {\em Journal of Machine Learning Research}, 23(226):1--61, 2022.

\bibitem{du2019gradient}
Simon~S. Du, Xiyu Zhai, Barnabas Poczos, and Aarti Singh.
\newblock Gradient descent provably optimizes over-parameterized neural networks, 2019.

\bibitem{fisher2019all}
Aaron Fisher, Cynthia Rudin, and Francesca Dominici.
\newblock All models are wrong, but many are useful: Learning a variable's importance by studying an entire class of prediction models simultaneously.
\newblock {\em Journal of Machine Learning Research}, 20(177):1--81, 2019.

\bibitem{kim2021visual}
Gwangsu Kim, Jaeson Jang, Seungdae Baek, Min Song, and Se-Bum Paik.
\newblock Visual number sense in untrained deep neural networks.
\newblock {\em Science advances}, 7(1):eabd6127, 2021.

\bibitem{liu2017learning}
Zhuang Liu, Jianguo Li, Zhiqiang Shen, Gao Huang, Shoumeng Yan, and Changshui Zhang.
\newblock Learning efficient convolutional networks through network slimming.
\newblock In {\em Proceedings of the IEEE international conference on computer vision}, pages 2736--2744, 2017.

\bibitem{madras2019detecting}
David Madras, James Atwood, and Alexander D'Amour.
\newblock Detecting extrapolation with local ensembles.
\newblock In {\em International Conference on Learning Representations}, 2019.

\bibitem{raskutti2014early}
Garvesh Raskutti, Martin~J Wainwright, and Bin Yu.
\newblock Early stopping and non-parametric regression: an optimal data-dependent stopping rule.
\newblock {\em The Journal of Machine Learning Research}, 15(1):335--366, 2014.

\bibitem{semenova2023a}
Lesia Semenova, Harry Chen, Ronald Parr, and Cynthia Rudin.
\newblock A path to simpler models starts with noise.
\newblock {\em Advances in Neural Information Processing Systems}, 36, 2024.

\bibitem{semenova_rudin}
Lesia Semenova, Cynthia Rudin, and Ronald Parr.
\newblock On the existence of simpler machine learning models.
\newblock In {\em Proceedings of the 2022 ACM Conference on Fairness, Accountability, and Transparency}, pages 1827--1858, 2022.

\bibitem{srebro2010smoothness}
Nathan Srebro, Karthik Sridharan, and Ambuj Tewari.
\newblock Smoothness, low noise and fast rates.
\newblock {\em Advances in neural information processing systems}, 23, 2010.

\bibitem{srivastava2014dropout}
Nitish Srivastava, Geoffrey Hinton, Alex Krizhevsky, Ilya Sutskever, and Ruslan Salakhutdinov.
\newblock Dropout: a simple way to prevent neural networks from overfitting.
\newblock {\em The journal of machine learning research}, 15(1):1929--1958, 2014.

\bibitem{Vapnik1998}
Vladimir~N. Vapnik.
\newblock {\em Statistical Learning Theory}.
\newblock Wiley-Interscience, 1998.

\end{thebibliography}
\end{document}